\documentclass{article}

\usepackage[preprint]{neurips_2026}
\usepackage[utf8]{inputenc} % allow utf-8 input
\usepackage[T1]{fontenc}    % use 8-bit T1 fonts
\usepackage{hyperref}       % hyperlinks
\usepackage{url}            % simple URL typesetting
\usepackage{booktabs}       % professional-quality tables
\usepackage{amsfonts}       % blackboard math symbols
\usepackage{nicefrac}       % compact symbols for 1/2, etc.
\usepackage{microtype}      % microtypography
\usepackage{xcolor}         % colors
\usepackage{url}
\usepackage{graphicx}
\usepackage{amsthm}
\usepackage{amsmath,amssymb,amsfonts}

\newtheorem{proposition}{Proposition}

\usepackage{graphicx}  % 插入图片
\usepackage{subcaption}
\usepackage[most]{tcolorbox} % 加载tcolorbox并启用most库，支持所有常用功能
\usepackage{lipsum} % 用于生成测试文本
\usepackage{hyperref}
\usepackage{url}
\usepackage{float}
\usepackage{graphicx}
\usepackage{booktabs}
\usepackage[most]{tcolorbox}
\usepackage{xcolor}
\usepackage{wrapfig}
\usepackage{bm}
\usepackage{multirow}
\usepackage{amsmath}
\usepackage{algorithm}
\usepackage{amssymb}
\usepackage{caption}
\usepackage{colortbl}
\usepackage{algpseudocode}
\usepackage{subcaption}
\usepackage{enumitem}
\usepackage{sidecap}
\usepackage{rotating}
\usepackage{makecell}
\usepackage{pifont}
\usepackage{caption}
\usepackage{subcaption}
\usepackage{tabularx}

\newtheorem{lemma}{Lemma}

\usepackage{pifont}% 

\newcommand{\wt}{\widetilde}

\renewcommand{\tilde}{\wt}

\definecolor{author-color}{RGB}{12, 136, 67}
\definecolor{hidden-blue}{RGB}{194,232,247}
\definecolor{hidden-orange}{RGB}{243,202,120}
\definecolor{hidden-green}{RGB}{34,139,34}
\definecolor{hidden-pink}{RGB}{255,245,247}
\definecolor{hidden-black}{RGB}{20,68,106}
\definecolor{purple}{RGB}{144,153,196}
\definecolor{yellow}{RGB}{255,228,123}
\definecolor{tkcolor}{RGB}{189, 219, 159}
\definecolor{nicegreen}{RGB}{40, 180, 80}
\definecolor{nicered}{RGB}{200, 0, 0}

\newtcolorbox{takeaways}[2][]{
    width=\columnwidth,
    toprule=0.0pt,ß
    leftrule=0.9pt,
    bottomrule=0.9pt,
    rightrule=0.9pt,
    arc=0pt,
    colback = tkcolor!10, 
    colframe = tkcolor, 
    boxsep=0pt,left=7pt,right=7pt,top=4pt,bottom=4pt,
    fontupper=\linespread{0.1}\selectfont,
    title=#2,#1,
    before skip=0.7em,  % 盒子上方的固定间距（覆盖默认规则）
    after skip=0.7em
}

\title{The Missing Piece in Pre-trained Model Evaluation: Reward-Guided Decoding Unlocks Task-Oriented Behavior Without Parameter Updates}

\author{
  {\bf
  Shaobo Wang$^{{\color{author-color}\boldsymbol{1,2}}}$$^{\dagger}$\thanks{Equal Contribution.} \quad
    Guo Chen$^{{\color{author-color}\boldsymbol{1}}}$$^{*}$ \quad
    Ziyue Wang$^{{\color{author-color}\boldsymbol{1}}}$ \quad
    Zhengyang Tang$^{{\color{author-color}\boldsymbol{3}}}$  
  } \\ \vspace{5pt}
  {\bf \hspace{3pt}
  Qingyang Liu$^{{\color{author-color}\boldsymbol{1}}}$ \quad
    Xingzhang Ren$^{{\color{author-color}\boldsymbol{2}}}$
 \quad
  Dayiheng Liu$^{{\color{author-color}\boldsymbol{2}}}$ \quad
    Linfeng Zhang$^{{\color{author-color}\boldsymbol{1}}}$\thanks{Corresponding authors.}
  } \\ \vspace{1pt}
  {$^{\color{author-color}\boldsymbol{1}}$Shanghai Jiao Tong University \quad 
    $^{\color{author-color}\boldsymbol{2}}$Qwen Team, Alibaba Group
  } \\ \vspace{0pt}
  {$^{\color{author-color}\boldsymbol{3}}$The Chinese University of Hong Kong, Shenzhen
  }
}

\begin{document}

\maketitle

\begin{abstract}
With the rapid progress of large language models (LLMs), reliably evaluating the capabilities of pre-trained LLMs has become increasingly important. The challenge is that base pre-trained models are optimized for next-token prediction and often fail to follow instructions or produce well-formed answers under standard prompting and direct decoding. As a result, benchmark performance can conflate model capability with decoding-induced failures to produce task-oriented outputs, while exposing such behavior often relies on costly post-training. Recent decoding-only approaches attempt to reshape output distributions, but such methods can be inefficient and brittle across open-ended tasks. To address these limitations, we propose \textbf{Energy-Based Decoding (EBD)}, a training-free, reward-guided framework for activating task-oriented behaviors from frozen pre-trained LLMs across both open-ended and objective tasks. EBD augments decoding with an external lightweight reward model, steering generations toward high-utility responses while anchoring them to the pre-trained model prior through a reward-tilted target distribution. We show that EBD shifts base-model outputs toward more instruction-following behavior, increasing behavioral similarity to post-trained counterparts and enabling a fairer inference-time evaluation of accessible pre-trained-model behavior. Empirically, EBD outperforms baselines across five models and six benchmarks, improving Qwen3-8B-Base on AlpacaEval2.0 from $8.8$ to $44.5$, reducing Mistral-7B Math500 latency by $18.9\times$ relative to prior decoding work, and remaining robust to reward-model size.
\end{abstract} 

\section{Introduction}
\label{sec:intro}

\begin{figure}[tb!]
  \centering
  \vspace{-10pt}
  \includegraphics[width=0.99\columnwidth]{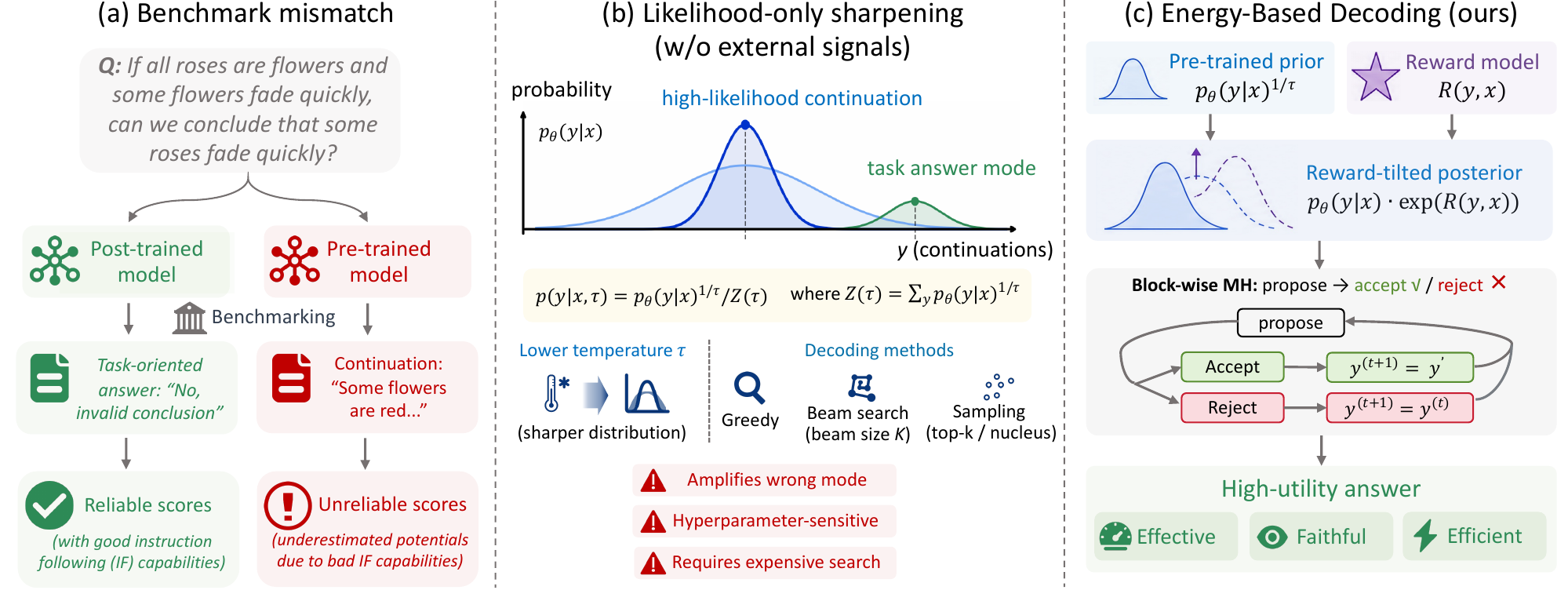}
  % \vspace{-3pt}
    \caption{\textbf{Steering, not sharpening}. 
    (a) Task queries induce a behavioral mismatch between pre-trained and post-trained LLMs: the latter interpret them as tasks, whereas the former tend to continue the prompt as plain text, yielding unreliable evaluation scores. 
    (b) Likelihood-only sharpening, including temperature reduction or search-based decoding, reshapes the prior without an external utility signal. It risks amplifying spurious high-likelihood modes, remains hyperparameter-sensitive, and entails expensive search, especially in open-ended tasks. 
    (c) EBD combines the pre-trained prior with a reward model to form a reward-tilted posterior, then applies block-wise Metropolis-Hastings updates to accept reward-improving proposals while anchoring generations to the prior.}
  \label{fig:motivation}
  \vspace{-10pt}
\end{figure}

Large language models (LLMs) have made striking progress across many domains, including mathematical reasoning, code generation, open-ended dialogue, and long-form generation~\citep{brown2020languagemodelsfewshotlearners, chowdhery2022palmscalinglanguagemodeling, zhao2026surveylargelanguagemodels}. Yet for pre-trained LLMs, standard prompting and direct decoding often provide a misleading view of what the model can do~\citep{karan2026reasoning, wang2024cotnoprompt, yue2025rlvr}. Most contemporary benchmarks evaluate models through task-oriented prompts, implicitly assuming that the model will recognize the prompt as a task, enter a problem-solving mode, and produce a well-formed answer~\citep{chang2023surveyevaluationlargelanguage, huang2023reasoninglargelanguagemodels}. This assumption is natural for post-trained models, but it breaks down for pre-trained models, which are trained primarily as continuation predictors~\citep{munjal2026instructiontunedmodelsperformbetter, irsoy-etal-2025-improving, lingteam2025flopcountsscaling300b}. Figure~\ref{fig:motivation}(a) illustrates this benchmark mismatch. The same prompt can lead a post-trained model to produce a task-oriented answer but a pre-trained model to produce a continuation-like response. As a result, poor direct-decoding performance may reflect a mismatch between continuation-style decoding and task-oriented evaluation, not simply the absence of task-relevant knowledge. \textit{\textbf{This shifts the central question toward whether inference-time decoding can activate task-oriented behavior from frozen pre-trained LLMs, enabling a fairer evaluation of what is accessible without parameter updates.}}

Activating task-oriented generation from a pre-trained model cannot be achieved reliably by simply drawing more samples from the same continuation-biased distribution. Instead, decoding must steer continuation behavior toward intentional answering while preserving the original model distribution. Recent decoding-only interventions~\citep{karan2026reasoning, wang2024cotnoprompt, wu-etal-2025-improve-decoding, nguyen2025turningheatminpsampling} suggest that reshaping the output distribution can alter model behavior and sometimes improve reasoning performance without modifying parameters. As summarized in Figure~\ref{fig:motivation}(b), however, likelihood-only shaping lacks an explicit, externally controllable utility signal. This is especially limiting for pre-trained models, whose prior is biased toward plausible continuation rather than task completion, so sharpening or truncating it can amplify the wrong mode, degrade open-ended generation quality, or require task-specific verifiers~\citep{hongler2026cognitivetraininglanguagemodels}. A closer family of baselines is Best-of-$N$ sampling or reranking, which draws independent responses and selects one using a score or task evaluator; however, repeated independent sampling can remain inefficient when most samples are continuation-like. Moreover, sample-heavy search can incur substantial computational overhead, making it impractical as a general-purpose way to improve pre-trained model behavior~\citep{wang2025efficientagentsbuildingeffective, liang2025plantainplananswerinterleavedreasoning}. \textit{\textbf{What is needed is an inference-time mechanism that can activate high-utility, task-oriented answers while remaining anchored to the pre-trained model and efficient enough for benchmark-scale evaluation.}}

To address these limitations, we propose \textbf{Energy-Based Decoding (EBD)}, a training-free inference framework for reward-guided decoding of pre-trained LLMs (Figure~\ref{fig:motivation}(c)). EBD augments decoding with an external lightweight reward model, which activates task-oriented, high-utility responses without updating model parameters. Critically, EBD is designed around three desiderata: (i) \textit{Effectiveness:} the induced distribution should closely match the reward-tilted target posterior; (ii) \textit{Faithfulness:} generation should remain anchored to the pre-trained prior, with KL drift controlled by the reward temperature and reward scale; and (iii) \textit{Efficiency:} the additional inference overhead should remain small enough for benchmark-scale use.

To realize these desiderata, we instantiate EBD as a block-wise Metropolis--Hastings (MH) sampler over model outputs. Starting from an initial response, EBD repeatedly selects a block, keeps the prefix fixed, and regenerates the remaining suffix from the base model's conditional distribution. The proposed response is then accepted or rejected by an MH rule guided by the reward model. The key design choice is that matched conditional-prior proposals and sequence-independent block selection ensure the proposal distribution equals the pre-trained conditional prior. As a result, the prior and proposal likelihood terms cancel in the acceptance ratio, and we call this proposal-prior cancellation. Under these conditions, each refinement step reduces to a reward comparison, while the Markov chain leaves the reward-tilted posterior invariant. This connects the algorithm to the three desiderata, since posterior targeting supports effectiveness, conditional generation from the pre-trained model anchors proposals to the prior, and cancellation keeps updates efficient. Contributions are as follows:

\begin{itemize}[leftmargin=*, topsep=2pt, itemsep=1pt, parsep=0pt, partopsep=0pt]
    \item We formulate reward-guided decoding as sampling from a reward-tilted posterior for activating task-oriented behavior under a frozen pre-trained prior. Under matched conditional-prior proposals, EBD enables proposal-prior cancellation, reducing each update to a reward comparison while controlling prior drift through the posterior temperature.
    \item We find that EBD does not merely improve aggregate scores; it shifts base-model behavior toward post-trained counterparts. Across five base-instruct pairs, behavioral similarity to the instruct model increases from $0.256$ under Direct to $0.385$ under EBD on average, a gain of $+0.129$.
    \item We evaluate EBD across five pre-trained models and six benchmarks. EBD consistently and substantially improves over Direct and previous decoding-based methods~\citep{karan2026reasoning} across all settings. On AlpacaEval2.0~\citep{dubois2023alpacafarm}, Qwen3-8B-Base~\citep{yang2025qwen3technicalreport} rises from $8.8$ to $44.5$, a $+405\%$ improvement in instruction-following quality; on Qwen2.5-7B~\citep{qwen2025qwen25technicalreport}, average objective score increases from $0.398$ to $0.557$; and on Mistral-7B Math500, EBD reduces per-question latency by $18.9\times$ relative to prior decoding work.
\end{itemize}

\section{Related Work}
  \label{sec:related}

\textbf{Capability evaluation of pre-trained LLMs.}
Modern benchmark evaluation often assumes that models have already been adapted to follow instructions through Supervised Fine-Tuning (SFT) or Reinforcement Learning from Human Feedback (RLHF)~\citep{brown2020languagemodelsfewshotlearners, ouyang2022instructgpt, munjal2026instructiontunedmodelsperformbetter}. For pre-trained large language models (LLMs), this assumption can be problematic because direct decoding may produce continuation-style text rather than task-oriented answers. Recent work shows that prompting and decoding choices can substantially change observed reasoning and task-solving behavior without parameter updates~\citep{wang2024cotnoprompt, yue2025rlvr, brown2024monkeys}. Thus, evaluating pre-trained models with direct decoding alone conflates task knowledge with the decoding procedure's ability to produce usable answers. Conversely, relying on post-training makes it difficult to separate behaviors inherited from pre-training from behaviors introduced by parameter updates. This leaves open the problem of improving task-oriented outputs from frozen pre-trained models at inference time.

\textbf{Inference-time decoding and search.}
A broad line of work improves model outputs by allocating additional inference-time computation, including self-consistency decoding~\citep{wang2023selfconsistency}, verifier or reward-model reranking~\citep{cobbe2021verifiers, lightman2024lets}, pass@$k$ sampling analyses~\citep{snell2024scaling, brown2024monkeys}, and sequential Markov chain Monte Carlo (MCMC) sampling from sharpened base-model distributions~\citep{karan2026reasoning}. These methods show that decoding can substantially change observed model performance, especially on reasoning tasks where high-quality trajectories may be missed by greedy or direct decoding. Best-of-$N$ sampling and reward reranking are particularly relevant baselines, since they select among independent prior samples and their effectiveness depends on drawing a good answer from the continuation-biased distribution. More generally, sample-heavy search and reranking can incur prohibitive cost, and verifier-based methods rely on reliable supervision that is often unavailable for open-ended generation. Moreover, likelihood-only sharpening uses the model's own prior as the sole steering signal, which can amplify continuation modes rather than task-oriented answers. This motivates decoding mechanisms that incorporate an external utility signal while controlling drift from the pre-trained prior.

\textbf{Reward-guided and energy-based generation.}
Reward-guided language generation is closely related to distributional control and energy-based modeling. Prior work formulates controlled generation as sampling from reward- or constraint-tilted distributions~\citep{khalifa2021gdc, korbak2022rlkl}, studies reward overoptimization under aggressive reward maximization~\citep{gao2023scaling}, and develops energy-based or Markov chain methods for controllable text generation, including residual energy-based models, energy language models, continuous Langevin decoding, constrained Metropolis--Hastings generation, and block Metropolis--Hastings samplers~\citep{deng2020residual, mireshghallah2022mixmatchlearningfreecontrollable, qin2022cold, miao2019cgmh, forristal2023blockmetropolishastingssamplercontrollable}. However, many of these approaches either train additional energy models, target hard lexical or style constraints, operate through continuous relaxations, or do not address reward-guided decoding for frozen pre-trained LLMs under benchmark-scale efficiency constraints. EBD instead treats the frozen pre-trained model as the prior and an external reward model as the energy term of an inference-time target distribution; when proposals exactly match the model's conditional prior, the MH ratio cancels the prior and proposal likelihoods, yielding efficient reward-based refinement while keeping the target distribution KL-regularized by the prior.

\begin{figure}[tb!]
  \centering
  \vspace{-5pt}
  \includegraphics[width=0.99\columnwidth]{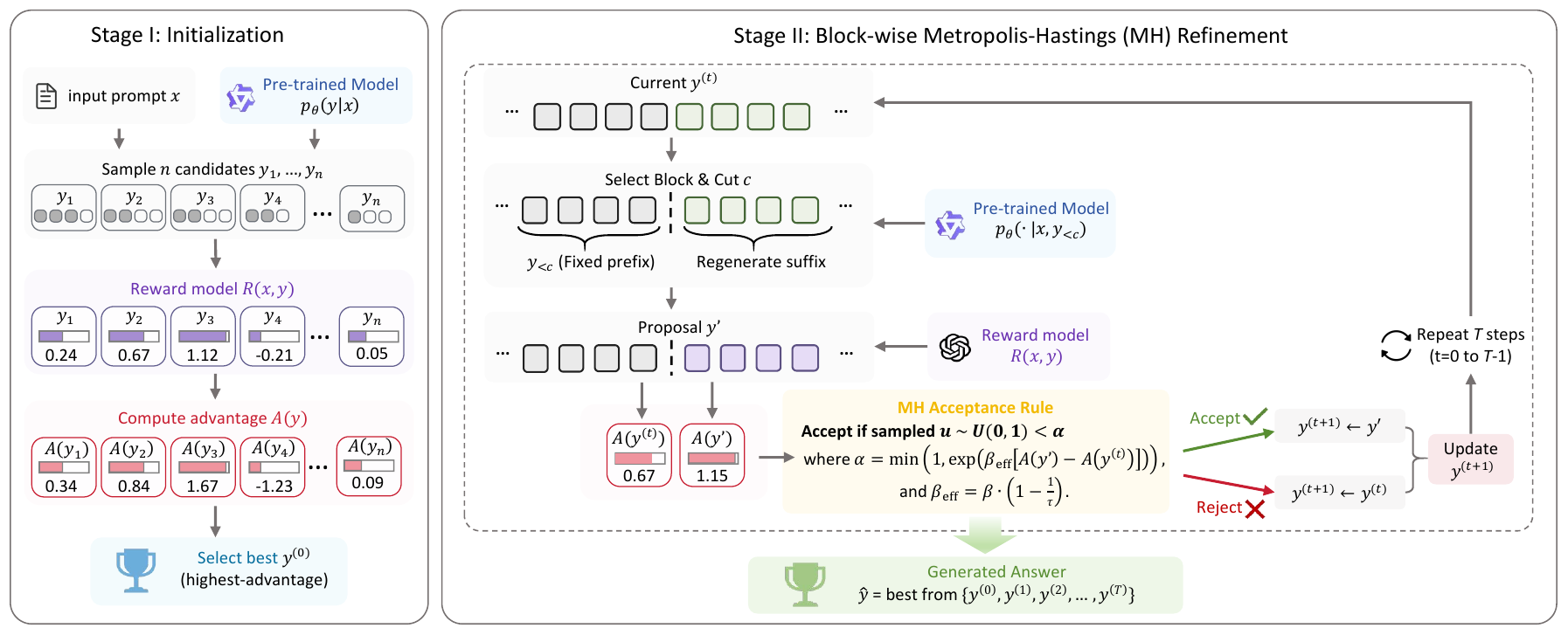}
  % \vspace{-7pt}
  \caption{EBD's pipeline follows two stages. \textbf{Stage I: Initialization} samples a small initialization pool from the pre-trained model, scores each response with the reward model, computes the prompt-level standardized advantage $\mathcal{A}_{\mathbf{x}}(\mathbf{y})=(R(\mathbf{y},\mathbf{x})-\mu_R(\mathbf{x}))/\sigma_R(\mathbf{x})$, and selects the best-in-pool response as the initial state. \textbf{Stage II: Block-wise MH refinement} repeatedly samples a cut, keeps the prefix, regenerates the suffix from the matched pre-trained conditional prior, scores the proposal, and accepts or rejects it using the reward-difference MH rule.}
  \label{fig:ebd_intro}
  \vspace{-10pt}
\end{figure}

\section{Method: Energy-Based Decoding}
  \label{sec:method}

\subsection{Preliminaries}

We consider an evaluation prompt set $\mathcal{D}$. For each prompt $\mathbf{x}\in\mathcal{D}$, $\mathcal{Y}(\mathbf{x})$ is the space of complete responses and $\mathbf{y}\in\mathcal{Y}(\mathbf{x})$ is one candidate response. A frozen pre-trained language model, together with a fixed decoding configuration, induces a full-response prior $p_\theta(\mathbf{y}\mid\mathbf{x})$. Operationally, this prior is the distribution obtained by running the base model with the same sampling temperature, stopping rule, and length limit used by the decoder. Because the prior is purely a sampling distribution, the decoder only needs to sample from it and its conditional suffix distribution, requiring no fine-tuning or training-time data access. The reward model returns a raw scalar score $R(\mathbf{y},\mathbf{x})\in\mathbb{R}$ for a complete prompt-response pair. The decoder then uses a fixed prompt-level score $S_{\mathbf{x}}(\mathbf{y})$; in our implementation this is a prompt-normalized version of $R$, defined in Eq.~\eqref{eq:advantage_score}. For the distribution-level discussion below, $q(\mathbf{y}\mid\mathbf{x})$ denotes a generic response distribution and $\mathcal{Q}$ denotes the family of distributions induced by practical decoding procedures.

A useful way to separate the design choices is to distinguish \emph{reward maximization} from \emph{distribution matching}. From an RL perspective, a complete response $\mathbf{y}$ can be viewed as the action and $S_{\mathbf{x}}(\mathbf{y})$ as a sequence-level reward. A pure reward-maximizing decoder would seek a distribution $q\in\mathcal{Q}$ that maximizes $\mathbb{E}_{\mathbf{y}\sim q}[S_{\mathbf{x}}(\mathbf{y})]$, analogous to PPO/GRPO-style reward optimization~\citep{stiennon2020summarize,ouyang2022instructgpt,shao2024deepseekmath}. This is an optimization or search view; it asks only ``which output has the largest reward?'' Without an explicit prior constraint, such a decoder can concentrate probability on low-prior reward artifacts or unnatural continuations.

EBD instead asks for a distribution that is close to the pre-trained model but tilted toward higher reward. Specifically, we define the decoding target through the KL-regularized objective
\begin{equation}
    \max_{q\in\Delta(\mathcal{Y}(\mathbf{x}))}
    \mathbb{E}_{\mathbf{y}\sim q(\cdot\mid\mathbf{x})}\!\left[S_{\mathbf{x}}(\mathbf{y})\right]
    -\frac{1}{\beta}\mathcal{D}_{\mathrm{KL}}\!\left(q(\cdot\mid\mathbf{x})\|p_\theta(\cdot\mid\mathbf{x})\right),
    \label{eq:kl_regularized_rl}
\end{equation}
where $\beta>0$ controls the reward-prior trade-off. When $\beta$ is small, the KL term dominates and the target stays close to direct base-model decoding; when $\beta$ is large, reward differences have stronger influence. Thus, the goal is not to find a single global reward maximizer, but to define a full response distribution that remains anchored to the frozen prior while reweighting outputs by utility. The unconstrained optimizer is the reward-tilted target
\begin{equation}
    \pi_\beta^*(\mathbf{y}\mid\mathbf{x})=\frac{1}{Z_\beta(\mathbf{x})}p_\theta(\mathbf{y}\mid\mathbf{x})\exp\!\left(\beta S_{\mathbf{x}}(\mathbf{y})\right),
    \label{eq:target_posterior}
\end{equation}
with $Z_\beta(\mathbf{x})=\mathbb{E}_{\tilde{\mathbf{y}}\sim p_\theta(\cdot\mid\mathbf{x})}[\exp(\beta S_{\mathbf{x}}(\tilde{\mathbf{y}}))]$. Eq.~\eqref{eq:target_posterior} is analogous to a Bayesian posterior, where the pre-trained model is the base measure and $\exp(\beta S_{\mathbf{x}}(\mathbf{y}))$ acts like a reward-derived likelihood factor. The normalizer $Z_\beta(\mathbf{x})$ is intractable but unnecessary for Metropolis--Hastings ratios.

Equivalently, the same target can be written as an energy-based distribution,
\begin{equation}
    \pi_\beta^*(\mathbf{y}\mid\mathbf{x})\propto \exp\!\left(-E_{\mathbf{x}}(\mathbf{y})\right),
    \qquad
    E_{\mathbf{x}}(\mathbf{y})=-\log p_\theta(\mathbf{y}\mid\mathbf{x})-\beta S_{\mathbf{x}}(\mathbf{y}).
    \label{eq:energy}
\end{equation}
The energy decomposes into two intuitive terms. The first, $-\log p_\theta(\mathbf{y}\mid\mathbf{x})$, penalizes responses that are implausible under the frozen model, while $-\beta S_{\mathbf{x}}(\mathbf{y})$ gives an energy bonus to responses preferred by the reward model. The sign convention is important, since increasing prior probability lowers the first term, and increasing reward lowers the second term. Low-energy responses are therefore both prior-plausible and high-utility. This is the central distinction from pure reward maximization, namely that distribution matching preserves the pre-trained prior as the reference distribution and applies reward only as an exponential tilt~\citep{korbak2022rlkl,khalifa2021gdc}, reducing the risk of unconstrained reward exploitation~\citep{gao2023scaling,skalse2022reward}. The algorithm below constructs a practical sampler for this target without exhaustive independent search.

\begin{algorithm}[tb!]
    \caption{Energy-Based Decoding (EBD) with Block-wise MH Refinement}\label{alg:ebd}
    \begin{algorithmic}
        \Statex \textbf{Input:} matched decoding prior $p_{\theta}$, reward model $R$, prompt $\mathbf{x}$, inverse temperature $\beta$, initialization pool size $n_{\mathrm{init}}$, refinement steps $K$, block count $M$, max length $L_{\max}$
        \State Sample initialization responses $\{\mathbf{y}_i\}_{i=1}^{n_{\mathrm{init}}}\sim p_\theta(\cdot\mid\mathbf{x})$
        \State Compute rewards $R_i=R(\mathbf{y}_i,\mathbf{x})$, mean $\mu_R(\mathbf{x})$, and standard deviation $\sigma_R(\mathbf{x})$
        \State Define $\mathcal{A}_{\mathbf{x}}(\mathbf{y})=(R(\mathbf{y},\mathbf{x})-\mu_R(\mathbf{x}))/\sigma_R(\mathbf{x})$
        \State Initialize $\mathbf{y}\leftarrow \mathop{\arg\max}_{\mathbf{y}_i}\mathcal{A}_{\mathbf{x}}(\mathbf{y}_i)$
        \For{$t = 0$ to $K-1$}
            \State Partition current response $\mathbf{y}$ into $M$ contiguous blocks
            \State Uniformly sample a block boundary / cut position $c$ from the admissible boundaries
            \State Propose a new suffix $\mathbf{y}'_{\ge c}\sim p_\theta(\cdot\mid\mathbf{x},\mathbf{y}_{<c})$
            \State Form proposal $\mathbf{y}'\leftarrow (\mathbf{y}_{<c},\mathbf{y}'_{\ge c})$, truncated by the stop rule or max length $L_{\max}$
            \State Compute $\mathcal{A}_{\mathbf{x}}(\mathbf{y}')$ using the fixed $\mu_R(\mathbf{x})$ and $\sigma_R(\mathbf{x})$
            \State $\alpha \leftarrow \min\left(1,\exp\left(\beta[\mathcal{A}_{\mathbf{x}}(\mathbf{y}')-\mathcal{A}_{\mathbf{x}}(\mathbf{y})]\right)\right)$
            \State Draw $u\sim\mathrm{Uniform}(0,1)$
            \If{$u\leq\alpha$}
                \State $\mathbf{y}\leftarrow\mathbf{y}'$
            \EndIf
        \EndFor
        \Statex \textbf{Output:} refined response $\mathbf{y}$
    \end{algorithmic}
\end{algorithm}

\subsection{The Energy-Based Decoding Algorithm}

Figure~\ref{fig:ebd_intro} summarizes EBD as a two-stage pipeline. Stage I initializes the chain by sampling, scoring, normalizing, and selecting an initial response. Stage II repeatedly refines that response by cutting it into a prefix and suffix, regenerating the suffix, scoring the proposal, and accepting or rejecting it. The main inputs are the frozen base model, the reward model, the prompt, the inverse temperature $\beta$, the initialization pool size $n_{\mathrm{init}}$, the number of refinement steps $K$, the number of blocks $M$, and the maximum generation length $L_{\max}$. Algorithm~\ref{alg:ebd} gives the corresponding pseudocode.

\noindent \textbf{Stage I: Initialization.}
EBD first samples an initialization pool $\{\mathbf{y}_i\}_{i=1}^{n_{\mathrm{init}}}$ from the frozen prior, $\mathbf{y}_i\sim p_\theta(\cdot\mid\mathbf{x})$. It then scores each response with the reward model, $R_i=R(\mathbf{y}_i,\mathbf{x})$, and computes prompt-level normalization statistics $\mu_R(\mathbf{x})$ and $\sigma_R(\mathbf{x})$ from this pool. Conditional on this initialization pool, the score is fixed during refinement. The implementation score is the standardized advantage
\begin{equation}
    \mathcal{A}_{\mathbf{x}}(\mathbf{y})=\frac{R(\mathbf{y},\mathbf{x})-\mu_R(\mathbf{x})}{\sigma_R(\mathbf{x})},
    \qquad S_{\mathbf{x}}(\mathbf{y})=\mathcal{A}_{\mathbf{x}}(\mathbf{y}).
    \label{eq:advantage_score}
\end{equation}
Statistics are fixed for the rest of decoding. EBD initializes the chain at the best response in the pool,
\begin{equation}
    \mathbf{y}^{(0)}=\mathop{\arg\max}_{\mathbf{y}_i}\mathcal{A}_{\mathbf{x}}(\mathbf{y}_i).
\end{equation}
This stage is a lightweight warm start rather than a large-$N$ reranking procedure; we use $n_{\mathrm{init}}=4$ by default. Importantly, the normalization is prompt-local, calibrating reward magnitudes for this prompt only and not introducing any training or dataset-level fitting.

\noindent \textbf{Stage II: Block-wise Metropolis-Hasting (MH) refinement.}
For refinement step $t=0,\ldots,K-1$, let $\mathbf{y}^{(t)}$ be the current full-response state. Operationally, a single refinement step proceeds by choosing a cut, keeping the prefix, sampling a new suffix, 
scoring the resulting full response, and finally accepting or rejecting the proposal.  EBD partitions the current token sequence into $M$ contiguous blocks and uses their boundaries as admissible cut positions. It then samples a cut position $c\sim P(c)$ from a fixed distribution, independent of the response content; in our implementation $P(c)$ is uniform over these admissible boundaries. EBD keeps the prefix before the cut and regenerates only the suffix from the same base-model conditional distribution used by direct decoding:
\begin{equation}
    \mathbf{y}'_{<c}=\mathbf{y}^{(t)}_{<c},
    \qquad
    \mathbf{y}'_{\ge c}\sim p_\theta(\cdot\mid\mathbf{x},\mathbf{y}^{(t)}_{<c}).
    \label{eq:block_proposal}
\end{equation}
The regenerated suffix is stopped by the usual end-of-sequence rule or by $L_{\max}$, producing a complete proposal $\mathbf{y}'$. The reward model then scores the full prompt-response pair $(\mathbf{x},\mathbf{y}')$, and the score is standardized using the same fixed $\mu_R(\mathbf{x})$ and $\sigma_R(\mathbf{x})$ from Stage I. Treating the sampled cut as an auxiliary variable, the joint proposal density is
\begin{equation}
    Q(c,\mathbf{y}'\mid\mathbf{y}^{(t)})=P(c)\,p_\theta(\mathbf{y}'_{\ge c}\mid\mathbf{x},\mathbf{y}^{(t)}_{<c}).
    \label{eq:joint_proposal}
\end{equation}
The key cancellation result is the following.
\begin{lemma}[Proposal-Prior Cancellation]
\label{lem:cancellation}
For the matched conditional-prior proposal in Eq.~\eqref{eq:joint_proposal}, the prior and proposal terms cancel for the sampled cut $c$:
\begin{equation}
\frac{p_\theta(\mathbf{y}'\mid\mathbf{x})\,Q(c,\mathbf{y}^{(t)}\mid\mathbf{y}')}
{p_\theta(\mathbf{y}^{(t)}\mid\mathbf{x})\,Q(c,\mathbf{y}'\mid\mathbf{y}^{(t)})}=1.
\end{equation}
\end{lemma}
The proof is given below. The practical implication is simple: although the target energy in Eq.~\eqref{eq:energy} contains the full-sequence prior term $-\log p_\theta(\mathbf{y}\mid\mathbf{x})$, EBD does not need to explicitly rescore full responses under the base model during refinement. Because the proposal suffix is sampled from the matched conditional prior, the prior probability terms cancel against the proposal terms. Combining Lemma~\ref{lem:cancellation} with the standard MH ratio for target $\pi_\beta^*$ yields the acceptance probability
\begin{equation}
\alpha(\mathbf{y}',c\mid\mathbf{y}^{(t)})=
\min\!\left(1,\exp\!\left(\beta\left[S_{\mathbf{x}}(\mathbf{y}')-S_{\mathbf{x}}(\mathbf{y}^{(t)})\right]\right)\right).
\label{eq:mh_accept_score}
\end{equation}
After drawing $u\sim\mathrm{Uniform}(0,1)$, EBD updates
\begin{equation}
\mathbf{y}^{(t+1)}=
\begin{cases}
\mathbf{y}', & u\le \alpha(\mathbf{y}',c\mid\mathbf{y}^{(t)}),\\
\mathbf{y}^{(t)}, & \text{otherwise}.
\end{cases}
\end{equation}
It repeats this cut-regenerate-score-accept procedure for $K$ steps and outputs the final full-response state $\mathbf{y}^{(K)}$. Thus, after the warm start, EBD is not a reranker over independent samples; it is a short Markov chain whose state is iteratively edited by suffix regeneration.

The cancellation is exact and complete for this matched setting, because the sampled cut is treated as an auxiliary variable, its sampling distribution is fully independent of response content, and the suffix proposal uses exactly the same conditional distribution as the prior term in Eq.~\eqref{eq:target_posterior}. If a state-dependent cutting heuristic or a different proposal is used, the corresponding proposal probabilities should be retained in the MH ratio. In the ideal matched-proposal setting, the chain leaves $\pi_\beta^*$ invariant asymptotically. In practice, EBD is a compute-bounded decoding procedure that runs only a small fixed number of refinement steps, trading suffix regenerations and reward-model evaluations for improved task utility while keeping the base model frozen.

% \subsection{Theoretical Properties}
% \label{subsec:theory}

The matched conditional-prior proposal yields a key efficiency property: although the target energy in Eq.~\eqref{eq:energy} contains the full-sequence prior $-\log p_\theta(\mathbf{y}\mid\mathbf{x})$, acceptance does not require explicitly rescoring it under the base model.

\begin{proposition}[Efficiency via Proposal-Prior Cancellation]
\label{prop:efficiency}
Under the matched conditional-prior proposal in Lemma~\ref{lem:cancellation}, the EBD acceptance step in Eq.~\eqref{eq:mh_accept_score} does not require explicit full-sequence likelihood re-scoring under $p_\theta(\mathbf{y}\mid\mathbf{x})$. The additional per-step cost reduces to one autoregressive suffix decode and one reward evaluation, scaling linearly with the number of refinement steps $K$.
\end{proposition}

\noindent \textit{Proof of Lemma~\ref{lem:cancellation}.}
Fix the sampled cut $c$ and write $\mathbf{y}$ for the current state $\mathbf{y}^{(t)}$. By construction, $\mathbf{y}_{<c}=\mathbf{y}'_{<c}$, so the autoregressive prior factors as
\begin{align*}
p_\theta(\mathbf{y}\mid\mathbf{x})
&=p_\theta(\mathbf{y}_{<c}\mid\mathbf{x})\,
p_\theta(\mathbf{y}_{\ge c}\mid\mathbf{x},\mathbf{y}_{<c}), \\
p_\theta(\mathbf{y}'\mid\mathbf{x})
&=p_\theta(\mathbf{y}_{<c}\mid\mathbf{x})\,
p_\theta(\mathbf{y}'_{\ge c}\mid\mathbf{x},\mathbf{y}_{<c}).
\end{align*}
Under the matched conditional-prior proposal, the forward and reverse proposal densities for the same auxiliary cut are
$Q(c,\mathbf{y}'\mid\mathbf{y})=P(c)\,p_\theta(\mathbf{y}'_{\ge c}\mid\mathbf{x},\mathbf{y}_{<c})$
and
$Q(c,\mathbf{y}\mid\mathbf{y}')=P(c)\,p_\theta(\mathbf{y}_{\ge c}\mid\mathbf{x},\mathbf{y}_{<c})$.
Substituting these factorizations into the Metropolis--Hastings ratio yields
\begin{equation*}
\frac{p_\theta(\mathbf{y}'\mid\mathbf{x})\,Q(c,\mathbf{y}\mid\mathbf{y}')}
{p_\theta(\mathbf{y}\mid\mathbf{x})\,Q(c,\mathbf{y}'\mid\mathbf{y})}=1.
\end{equation*}
The prior likelihoods and proposal densities exactly cancel. \qed

\noindent \textit{Proof of Proposition~\ref{prop:efficiency}.}
From Eq.~\eqref{eq:mh_accept_score} and Lemma~\ref{lem:cancellation}, the Metropolis--Hastings acceptance ratio reduces to the reward difference $\beta\bigl[S_{\mathbf{x}}(\mathbf{y}')-S_{\mathbf{x}}(\mathbf{y}^{(t)})\bigr]$. The full-sequence likelihood terms of the current and proposed responses cancel because the proposal is constructed from the same conditional distribution that defines the prior factor in $\pi_\beta^*$. Therefore EBD incurs no additional full-sequence likelihood pass during acceptance, in contrast to verifier-based reranking pipelines that explicitly evaluate candidate scores or likelihoods over an independent pool~\citep{cobbe2021verifiers,lightman2024lets}. The remaining overhead arises from proposal generation (one autoregressive decode of the resampled suffix per step) and from evaluating the reward function; both scale linearly in $K$. \qed

\section{Experiments}
\label{sec:experiments}

\subsection{Experimental Settings}

\paragraph{Models and baselines.}
We conducted experiments on a diverse suite of open-source large language models, specifically Meta-Llama-3-8B~\citep{grattafiori2024llama}, Mistral-7B-v0.3~\citep{jiang2023mistral7b}, Qwen2.5-7B~\citep{qwen2025qwen25technicalreport}, Qwen3-8B-Base~\citep{yang2025qwen3technicalreport}, and Olmo-3-1025-7B~\citep{olmo2026olmo3}.\footnote{The Qwen naming convention differs across releases: Qwen2.5-7B denotes the base checkpoint and Qwen2.5-7B-Instruct denotes the instruct-tuned checkpoint, whereas Qwen3-8B denotes the instruct-tuned checkpoint and Qwen3-8B-Base denotes the base checkpoint.} 
All models are of comparable scale, ranging from 7B to 8B parameters, and were evaluated under identical inference configurations unless otherwise noted. 
Our evaluation compared the proposed EBD framework against two baseline generation strategies, namely \textit{Direct} decoding and \textit{Power Sampling}~\citep{karan2026reasoning}.

\noindent \textbf{Benchmarks.} 
We evaluated the performance of pre-trained models on both objective and subjective tasks. The objective benchmarks consisted of GPQA~\citep{rein2023gpqa} for graduate-level reasoning, Math500~\citep{hendrycks2021math} for mathematical problem solving, and HumanEval~\citep{chen2021humaneval} for code generation. For subjective evaluation, we employed AlpacaEval2.0~\citep{dubois2023alpacafarm} for instruction following, MT-Bench~\citep{zheng2023mtbench} for multi-turn conversation quality, and WritingBench~\citep{wu2025writingbench} for open-ended writing proficiency.

\noindent \textbf{Subjective metrics.} 
For AlpacaEval2.0, we utilized the benchmark standard score where higher values indicated better performance. MT-Bench measured the average judge score on a scale from 1 to 10 across multiple turns and categories. WritingBench aggregated task-specific scores into an overall benchmark score to assess open-ended writing proficiency.

\noindent \textbf{Implementation details of EBD.} 
Unless otherwise specified, EBD employed a unified set of hyperparameters across all experiments, namely inverse temperature $\beta=3.5$, sampling temperature $\tau=1$, $K=12$ refinement steps, initialization pool size $n_{\mathrm{init}}=4$, block count $M=12$, and maximum generation length $L_{\max}=3072$ tokens. For all sampling-based methods, decoding was performed with sampling enabled at temperature $\tau=1$.

\noindent \textbf{Compute.} All experiments run on a cluster of 8 NVIDIA A100-SXM4-80GB GPUs with NCCL distributed backend for efficient data-parallel inference. Each GPU independently processes its assigned subset of evaluation samples; the benchmark is partitioned across the 8 GPUs without model sharding or pipeline parallelism. Latency is reported as the per-question average wall-clock time across all GPU workers, including both base-model generation and reward-model forward passes. VRAM cost is reported as peak memory allocation per GPU during inference. The core decoding loop of EBD can execute on a single A100-80GB GPU, since the per-sample computational graph does not require any inter-GPU communication.

\noindent \textbf{Software Environment.}
Our implementation is developed using Python, built on top of PyTorch 2.x and HuggingFace Transformers 4.x libraries.
All pre-trained language models are loaded in bfloat16 floating-point precision for efficient computation, and \texttt{trust\_remote\_code} is enabled as required for compatibility.
To ensure full experimental reproducibility across all runs, we set the global random seed to a fixed value of 42.
Notably, we do not employ vLLM-accelerated batched inference in our experiments; instead, all generation and decoding processes are conducted in a standard autoregressive manner using native HuggingFace framework interfaces.

\subsection{Main Results}

\begin{table}[tb!]
    \centering
    \vspace{-5pt}
    \caption{Performance comparison across objective and subjective benchmarks. Superscripts denote absolute change relative to + Direct within each backbone.}
    \label{tab:main_results_full}
    \vspace{3pt}
    \scriptsize
    \setlength{\tabcolsep}{2.4pt}
    \resizebox{\textwidth}{!}{%
    \begin{tabular}{l c c c c c c c c}
        \toprule
        \multirow{2}{*}{Method} & \multicolumn{4}{c}{\textbf{Objective Benchmarks}} & \multicolumn{4}{c}{\textbf{Subjective Benchmarks}} \\
        \cmidrule(lr){2-5} \cmidrule(lr){6-9}
        & Avg. & GPQA & Math500 & HumanEval & Avg. & Alpaca & MT-Bench & Writing \\ 
        \midrule
        %\rowcolor{gray!15}
        \multicolumn{9}{l}{\textit{Backbone: Llama3-8B}} \\ \midrule
        + Direct & 
        0.155\phantom{\textcolor{nicegreen}{$^{+0.000}$}} & 0.152\phantom{\textcolor{nicegreen}{$^{+0.000}$}} & 
        0.016\phantom{\textcolor{nicegreen}{$^{+0.000}$}} & 
        0.299\phantom{\textcolor{nicegreen}{$^{+0.000}$}} & 
        2.078\phantom{\textcolor{nicegreen}{$^{+0.000}$}} & 
        1.525\phantom{\textcolor{nicegreen}{$^{+0.000}$}} & 
        3.000\phantom{\textcolor{nicegreen}{$^{+0.000}$}} & 
        1.709\phantom{\textcolor{nicegreen}{$^{+0.000}$}} \\
        + Power Sampling & 0.195\textcolor{nicegreen}{$^{+0.040}$} & 0.157\textcolor{nicegreen}{$^{+0.005}$} & 0.008\textcolor{nicered}{$^{-0.008}$} & 0.421\textcolor{nicegreen}{$^{+0.122}$} & 1.813\textcolor{nicered}{$^{-0.265}$} & 2.342\textcolor{nicegreen}{$^{+0.817}$} & 1.889\textcolor{nicered}{$^{-1.111}$} & 1.208\textcolor{nicered}{$^{-0.501}$} \\
        \rowcolor{cyan!10}
        \textbf{+ EBD (ours)} & \textbf{0.262}\textcolor{nicegreen}{$^{+0.107}$} & \textbf{0.237}\textcolor{nicegreen}{$^{+0.085}$} & \textbf{0.042}\textcolor{nicegreen}{$^{+0.026}$} & \textbf{0.506}\textcolor{nicegreen}{$^{+0.207}$} & \textbf{3.258}\textcolor{nicegreen}{$^{+1.180}$} & \textbf{3.862}\textcolor{nicegreen}{$^{+2.337}$} & \textbf{3.683}\textcolor{nicegreen}{$^{+0.683}$} & \textbf{2.230}\textcolor{nicegreen}{$^{+0.521}$} \\
        \midrule
        %\rowcolor{gray!15}
        \multicolumn{9}{l}{\textit{Backbone: Mistral-7B}} \\
        \midrule
        + Direct & 
        0.108\phantom{\textcolor{nicegreen}{$^{+0.000}$}} & 
        0.167\phantom{\textcolor{nicegreen}{$^{+0.000}$}} & 
        0.010\phantom{\textcolor{nicegreen}{$^{+0.000}$}} & 
        0.147\phantom{\textcolor{nicegreen}{$^{+0.000}$}} & 
        2.294\phantom{\textcolor{nicegreen}{$^{+0.000}$}} & 
        2.153\phantom{\textcolor{nicegreen}{$^{+0.000}$}} & 
        2.868\phantom{\textcolor{nicegreen}{$^{+0.000}$}} & 
        1.860\phantom{\textcolor{nicegreen}{$^{+0.000}$}} \\
        + Power Sampling & 0.126\textcolor{nicegreen}{$^{+0.018}$} & 0.182\textcolor{nicegreen}{$^{+0.015}$} & 0.006\textcolor{nicered}{$^{-0.004}$} & 0.189\textcolor{nicegreen}{$^{+0.042}$} & 1.713\textcolor{nicered}{$^{-0.581}$} & 1.600\textcolor{nicered}{$^{-0.553}$} & 2.294\textcolor{nicered}{$^{-0.574}$} & 1.245\textcolor{nicered}{$^{-0.615}$} \\
        \rowcolor{cyan!10}
        \textbf{+ EBD (ours)} & \textbf{0.168}\textcolor{nicegreen}{$^{+0.060}$} & \textbf{0.241}\textcolor{nicegreen}{$^{+0.074}$} & \textbf{0.032}\textcolor{nicegreen}{$^{+0.022}$} & \textbf{0.232}\textcolor{nicegreen}{$^{+0.085}$} & \textbf{4.776}\textcolor{nicegreen}{$^{+2.482}$} & \textbf{7.249}\textcolor{nicegreen}{$^{+5.096}$} & \textbf{4.182}\textcolor{nicegreen}{$^{+1.314}$} & \textbf{2.898}\textcolor{nicegreen}{$^{+1.038}$} \\
        \midrule
        %\rowcolor{gray!15}
        \multicolumn{9}{l}{\textit{Backbone: Qwen2.5-7B}} \\
        \midrule
        + Direct & 
        0.398\phantom{\textcolor{nicegreen}{$^{+0.000}$}} & 
        0.222\phantom{\textcolor{nicegreen}{$^{+0.000}$}} & 
        0.558\phantom{\textcolor{nicegreen}{$^{+0.000}$}} & 
        0.415\phantom{\textcolor{nicegreen}{$^{+0.000}$}} & 
        5.750\phantom{\textcolor{nicegreen}{$^{+0.000}$}} & 
        8.632\phantom{\textcolor{nicegreen}{$^{+0.000}$}} & 
        5.375\phantom{\textcolor{nicegreen}{$^{+0.000}$}} & 
        3.243\phantom{\textcolor{nicegreen}{$^{+0.000}$}} \\
        + Power Sampling & 0.513\textcolor{nicegreen}{$^{+0.115}$} & 0.303\textcolor{nicegreen}{$^{+0.081}$} & 0.664\textcolor{nicegreen}{$^{+0.106}$} & 0.573\textcolor{nicegreen}{$^{+0.158}$} & 4.670\textcolor{nicered}{$^{-1.080}$} & 7.511\textcolor{nicered}{$^{-1.121}$} & 4.600\textcolor{nicered}{$^{-0.775}$} & 1.898\textcolor{nicered}{$^{-1.345}$} \\
        \rowcolor{cyan!10}
        \textbf{+ EBD (ours)} & \textbf{0.557}\textcolor{nicegreen}{$^{+0.159}$} & \textbf{0.328}\textcolor{nicegreen}{$^{+0.106}$} & \textbf{0.720}\textcolor{nicegreen}{$^{+0.162}$} & \textbf{0.622}\textcolor{nicegreen}{$^{+0.207}$} & \textbf{10.728}\textcolor{nicegreen}{$^{+4.978}$} & \textbf{20.956}\textcolor{nicegreen}{$^{+12.324}$} & \textbf{6.905}\textcolor{nicegreen}{$^{+1.530}$} & \textbf{4.322}\textcolor{nicegreen}{$^{+1.079}$} \\
        \midrule
        %\rowcolor{gray!15}
        \multicolumn{9}{l}{\textit{Backbone: Qwen3-8B-Base}} \\
        \midrule
        + Direct & 
        0.544\phantom{\textcolor{nicegreen}{$^{+0.000}$}} & 
        0.343\phantom{\textcolor{nicegreen}{$^{+0.000}$}} & 
        0.690\phantom{\textcolor{nicegreen}{$^{+0.000}$}} & 
        0.598\phantom{\textcolor{nicegreen}{$^{+0.000}$}} & 
        6.707\phantom{\textcolor{nicegreen}{$^{+0.000}$}} & 
        8.811\phantom{\textcolor{nicegreen}{$^{+0.000}$}} & 
        6.115\phantom{\textcolor{nicegreen}{$^{+0.000}$}} & 
        5.196\phantom{\textcolor{nicegreen}{$^{+0.000}$}} \\
        + Power Sampling & 0.622\textcolor{nicegreen}{$^{+0.078}$} & 0.409\textcolor{nicegreen}{$^{+0.066}$} & \textbf{0.792}\textcolor{nicegreen}{$^{+0.102}$} & 0.665\textcolor{nicegreen}{$^{+0.067}$} & 12.990\textcolor{nicegreen}{$^{+6.283}$} & 29.219\textcolor{nicegreen}{$^{+20.408}$} & 6.043\textcolor{nicered}{$^{-0.072}$} & 3.707\textcolor{nicered}{$^{-1.489}$} \\
        \rowcolor{cyan!10}
        \textbf{+ EBD (ours)} & \textbf{0.632}\textcolor{nicegreen}{$^{+0.088}$} & \textbf{0.414}\textcolor{nicegreen}{$^{+0.071}$} & \textbf{0.792}\textcolor{nicegreen}{$^{+0.102}$} & \textbf{0.689}\textcolor{nicegreen}{$^{+0.091}$} & \textbf{19.360}\textcolor{nicegreen}{$^{+12.653}$} & \textbf{44.519}\textcolor{nicegreen}{$^{+35.708}$} & \textbf{7.700}\textcolor{nicegreen}{$^{+1.585}$} & \textbf{5.862}\textcolor{nicegreen}{$^{+0.666}$} \\
        \midrule
        %\rowcolor{gray!15}
        \multicolumn{9}{l}{\textit{Backbone: Olmo-7B}} \\
        \midrule
        + Direct & 
        0.329\phantom{\textcolor{nicegreen}{$^{+0.000}$}} & 
        0.303\phantom{\textcolor{nicegreen}{$^{+0.000}$}} & 
        0.366\phantom{\textcolor{nicegreen}{$^{+0.000}$}} & 
        0.317\phantom{\textcolor{nicegreen}{$^{+0.000}$}} & 
        8.044\phantom{\textcolor{nicegreen}{$^{+0.000}$}} & 
        14.558\phantom{\textcolor{nicegreen}{$^{+0.000}$}} & 
        6.294\phantom{\textcolor{nicegreen}{$^{+0.000}$}} & 
        3.279\phantom{\textcolor{nicegreen}{$^{+0.000}$}} \\
        + Power Sampling & 0.382\textcolor{nicegreen}{$^{+0.053}$} & 0.268\textcolor{nicered}{$^{-0.035}$} & 0.464\textcolor{nicegreen}{$^{+0.098}$} & 0.415\textcolor{nicegreen}{$^{+0.098}$} & 6.386\textcolor{nicered}{$^{-1.658}$} & 12.688\textcolor{nicered}{$^{-1.870}$} & 4.789\textcolor{nicered}{$^{-1.505}$} & 1.681\textcolor{nicered}{$^{-1.598}$} \\
        \rowcolor{cyan!10}
        \textbf{+ EBD (ours)} & \textbf{0.538}\textcolor{nicegreen}{$^{+0.209}$} & \textbf{0.343}\textcolor{nicegreen}{$^{+0.040}$} & \textbf{0.704}\textcolor{nicegreen}{$^{+0.338}$} & \textbf{0.567}\textcolor{nicegreen}{$^{+0.250}$} & \textbf{13.378}\textcolor{nicegreen}{$^{+5.334}$} & \textbf{28.238}\textcolor{nicegreen}{$^{+13.680}$} & \textbf{7.273}\textcolor{nicegreen}{$^{+0.979}$} & \textbf{4.624}\textcolor{nicegreen}{$^{+1.345}$} \\
        \bottomrule
    \end{tabular}%
    }
    \vspace{-8pt}
\end{table}

\begin{table}[tb!]
    \centering
    \vspace{-4pt}
    \caption{Average per-question inference time (seconds). Direct is shown as a latency reference; speedups compare EBD against Power Sampling. Lower time and higher speedup are better.}
    \vspace{3pt}
    \label{tab:efficiency_full}
    \small
    \setlength{\tabcolsep}{4pt}
    \resizebox{\textwidth}{!}{%
    \begin{tabular}{l c c c c c c c c}
        \toprule
        \multirow{2}{*}{Backbone} & \multicolumn{4}{c}{\textbf{Objective Avg.}} & \multicolumn{4}{c}{\textbf{Subjective Avg.}} \\
        \cmidrule(lr){2-5} \cmidrule(lr){6-9}
        & \textcolor{gray}{Direct (ref.)} & Power & EBD & Speedup & \textcolor{gray}{Direct (ref.)} & Power & EBD & Speedup \\
        \midrule
        Meta-Llama-3-8B  & \textcolor{gray}{6.17} & 382.61 & \textbf{80.47} & $4.8\times$  & \textcolor{gray}{15.37} & 1372.11 & \textbf{540.00} & $2.5\times$ \\
        Mistral-7B-v0.3 & \textcolor{gray}{8.17} & 1166.87 & \textbf{66.17} & $17.6\times$ & \textcolor{gray}{24.19} & 2294.71 & \textbf{653.77} & $3.5\times$ \\
        Qwen2.5-7B      & \textcolor{gray}{9.58} & 342.99 & \textbf{93.54} & $3.7\times$  & \textcolor{gray}{9.39} & 482.95 & \textbf{205.24} & $2.4\times$ \\
        Qwen3-8B-Base   & \textcolor{gray}{11.56} & 476.74 & \textbf{127.76} & $3.7\times$ & \textcolor{gray}{21.18} & 893.76 & \textbf{414.75} & $2.2\times$ \\
        Olmo-3-1025-7B  & \textcolor{gray}{28.23} & 1481.60 & \textbf{626.04} & $2.4\times$ & \textcolor{gray}{28.52} & 1557.84 & \textbf{901.69} & $1.7\times$ \\
        \bottomrule
    \end{tabular}%
    }
    \vspace{-15pt}
\end{table}

\noindent \textbf{Objective benchmarks.} Table~\ref{tab:main_results_full} shows that EBD improves objective-task performance across all five pre-trained models. Average objective scores rise from 0.155 to 0.262 for Meta-Llama-3-8B, from 0.398 to 0.557 for Qwen2.5-7B, and from 0.329 to 0.538 for Olmo-3-1025-7B. These gains indicate that direct decoding misses task-oriented reasoning and code-generation behavior that remains accessible under reward-guided, prior-preserving inference.

\noindent \textbf{Subjective benchmarks.} As shown in Table~\ref{tab:main_results_full}, the gains are larger on open-ended benchmarks, where continuation-style failures are especially harmful. EBD increases AlpacaEval2.0 from 2.153 to 7.249 on Mistral-7B-v0.3 and from 8.811 to 44.519 on Qwen3-8B-Base. On MT-Bench, it improves Qwen2.5-7B from 5.375 to 6.905 and Qwen3-8B-Base from 6.115 to 7.700. In contrast, Power Sampling frequently reduces subjective quality, suggesting that likelihood-only search can amplify the wrong generation mode rather than activate task-oriented responses.

\begin{figure}[tb!]
  \centering
  \vspace{10pt}
  \includegraphics[width=0.99\linewidth]{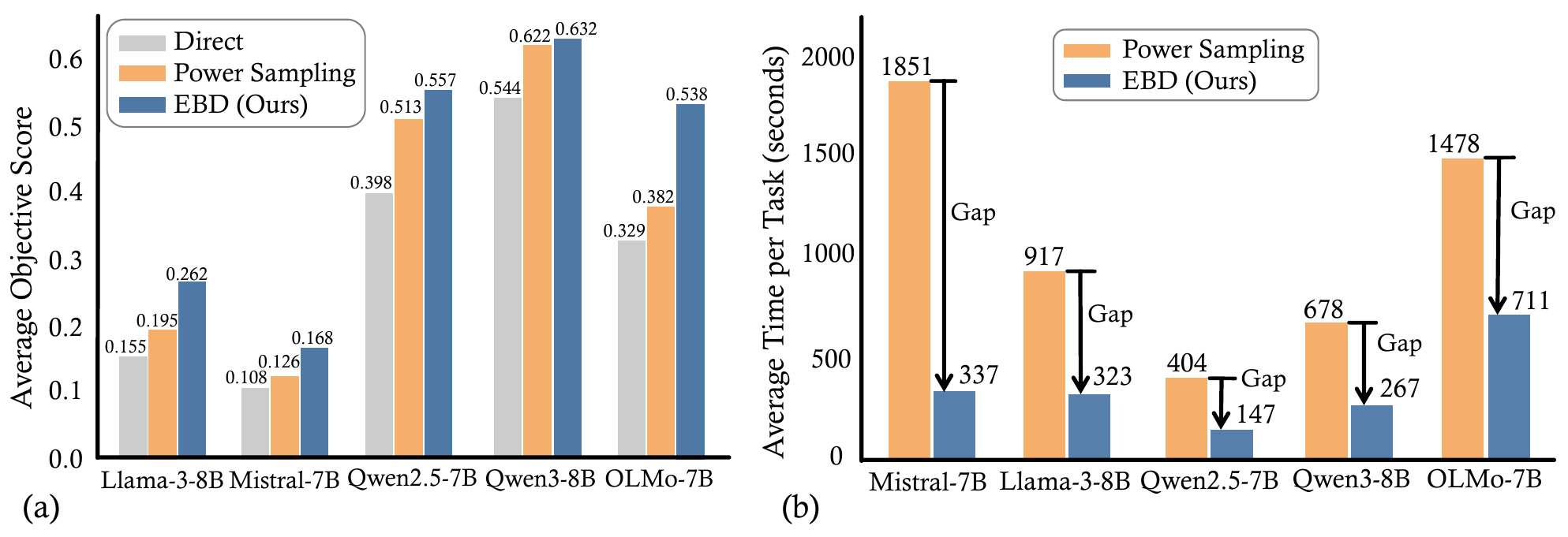}
  \vspace{-2pt}
  \caption{Quality and latency across five pre-trained base models.
  (a)~Average objective benchmark score.
  (b)~Average inference time per task.
  EBD consistently achieves the highest scores while remaining substantially faster than Power Sampling.}
  \label{fig:quality_efficiency_frontier}
  \vspace{-8pt}
\end{figure}

\noindent \textbf{Efficiency analysis.} Figures~\ref{fig:quality_efficiency_frontier}(a) and \ref{fig:quality_efficiency_frontier}(b) compare quality and latency. Direct decoding is fastest but often fails to produce usable task-oriented outputs, whereas Power Sampling is much slower and does not reliably improve subjective quality. EBD consistently achieves the highest scores while remaining substantially faster than Power Sampling, demonstrating a clear advantage. As summarized in Table~\ref{tab:efficiency_full}, EBD reduces average objective latency by $17.6\times$ over Power Sampling on Mistral-7B-v0.3.

\noindent \textbf{Behavioral alignment analysis.} We measure whether EBD activates effective success patterns that resemble post-trained variants by computing Pearson correlation of binary correctness over 1,000 samples. To further test the generality of this analysis, we additionally include Qwen2.5-14B~\citep{qwen2025qwen25technicalreport}, a larger model beyond the five main experimental backbones. Figure~\ref{fig:sample_level_shift} visualizes this alignment as stacked bars: the light segment shows the Direct baseline correlation with the instruct model, while the dark segment shows the steady EBD increment. Across all five model pairs, EBD raises the bars above their baselines, and Table~\ref{tab:sample_corr} confirms that the average correlation rises from 0.256 to 0.385. The largest shift occurs for Meta-Llama-3-8B, where correlation increases from 0.112 to 0.307, indicating that reward-guided refinement changes not only aggregate scores but also which questions the frozen model can answer under inference-time activation.

\begin{table}[tb!]
    \centering
    \vspace{-3pt}
    \begin{minipage}[t]{0.47\textwidth}
        \centering
        \vspace{0pt}
        \captionof{table}{Behavioral alignment with post-trained models. This table shows pearson correlation of binary correctness over 1,000 questions.}
        \label{tab:sample_corr}
        \scriptsize
        \setlength{\tabcolsep}{4.3pt}
        \begin{tabularx}{\linewidth}{@{}l *{4}{>{\centering\arraybackslash}X}@{}}
            \toprule
            Model & Direct & EBD & Abs. & Rel. gain \\
            \midrule
            Meta-Llama-3-8B  & 0.112 & 0.307 & \textcolor{nicegreen}{+0.195} & \textcolor{nicegreen}{+174.1\%} \\
            Mistral-7B-v0.3  & 0.212 & 0.327 & \textcolor{nicegreen}{+0.115} & \textcolor{nicegreen}{+54.2\%} \\
            Qwen2.5-7B       & 0.297 & 0.379 & \textcolor{nicegreen}{+0.082} & \textcolor{nicegreen}{+27.6\%} \\
            Qwen2.5-14B      & 0.309 & 0.399 & \textcolor{nicegreen}{+0.090} & \textcolor{nicegreen}{+29.1\%} \\
            Olmo-3-1025-7B   & 0.349 & 0.512 & \textcolor{nicegreen}{+0.164} & \textcolor{nicegreen}{+47.0\%} \\
            \midrule
            \rowcolor{cyan!10}
            \textbf{Average} & 0.256 & \textbf{0.385} & \textbf{\textcolor{nicegreen}{+0.129}} & \textbf{\textcolor{nicegreen}{+50.4\%}} \\
            \bottomrule
        \end{tabularx}
    \end{minipage}%
    \hfill
    \begin{minipage}[t]{0.50\textwidth}
        \centering
        \vspace{0pt}
        \captionof{table}{Format adherence and response validity on Math500. VRR measures the fraction of outputs with a boxed answer. D$\to$E is Direct to EBD.}
        \vspace{3pt}
        \label{tab:vrr_math500}
        \scriptsize
        \setlength{\tabcolsep}{2.2pt}
        \begin{tabularx}{\linewidth}{@{}l *{4}{>{\centering\arraybackslash}X}@{}}
            \toprule
            Backbone & VRR D$\to$E & $\Delta$ & Acc. D$\to$E & $\Delta$ \\
            \midrule
            Llama3-8B & $3.2{\to}28.5$ & \textcolor{nicegreen}{+25.3} & $1.6{\to}4.2$ & \textcolor{nicegreen}{+2.6} \\
            Qwen2.5-7B & $75.2{\to}80.9$ & \textcolor{nicegreen}{+5.7} & $62.8{\to}64.8$ & \textcolor{nicegreen}{+2.0} \\
            Llama3.1-8B & $3.5{\to}31.0$ & \textcolor{nicegreen}{+27.5} & $1.7{\to}2.4$ & \textcolor{nicegreen}{+0.7} \\
            Llama3.1-8B+Dolly & $51.0{\to}73.0$ & \textcolor{nicegreen}{+22.0} & $4.8{\to}13.6$ & \textcolor{nicegreen}{+8.8} \\
            Llama3.1-8B+Hybrid & $54.0{\to}72.0$ & \textcolor{nicegreen}{+18.0} & $4.6{\to}15.0$ & \textcolor{nicegreen}{+10.4} \\
            \bottomrule
        \end{tabularx}
    \end{minipage}
    \vspace{-15pt}
\end{table}

\noindent \textbf{Format adherence and response validity.} We next test whether EBD remains effective on instruction-tuned models. Valid Response Rate (VRR) measures the fraction of Math500 outputs that contain the required boxed answer. We construct a Hybrid dataset of 10k carefully curated high-quality samples drawn from GSM8K~\citep{cobbe2021verifiers}, MBPP~\citep{austin2021mbpp}, ARC-Easy~\citep{clark2018arc}, and Dolly~\citep{conover2023dolly}, then separately cold-start SFT Llama3.1-8B with Hybrid and Dolly. As shown in Table~\ref{tab:vrr_math500}, EBD consistently improves VRR across both initialization conditions, revealing its robustness to the underlying SFT recipe and the initial model state. Accuracy increases alongside validity, suggesting that EBD activates more parseable and task-aligned solution modes rather than merely changing the answer format.

\noindent \textbf{Case study.} 
Appendix~\ref{app:case_studies} provides 30 representative case studies. These examples consistently show that EBD transforms otherwise unusable or socially-styled direct outputs into structured, task-aligned responses with parseable final answers, demonstrating that the VRR gains reported in Table~\ref{tab:vrr_math500} reflect a systematic improvement rather than cherry-picking of easy questions.

\begin{figure}[tb!]
    \vspace{5pt}
    \begin{minipage}{0.7\textwidth}
        \vspace{0pt}
        \includegraphics[width=\linewidth]{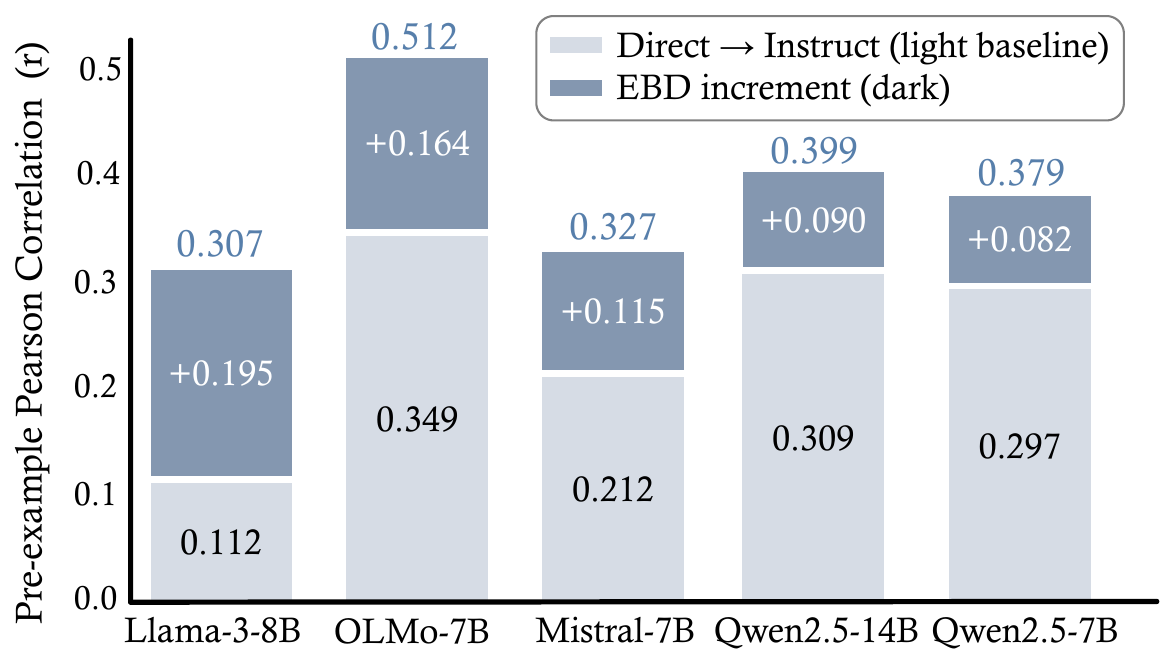}
    \end{minipage}%
    \hfill
    \begin{minipage}{0.27\textwidth}
        \vspace{0pt}
        \caption{Behavioral alignment between pre-trained models and their post-trained counterparts. The light segment shows the baseline correlation with the instruct model, while the dark segment shows the steady EBD increment. EBD consistently shifts the behavioral signature of pre-trained models toward post-trained outputs.}
        \label{fig:sample_level_shift}
    \end{minipage}
    \vspace{-7pt}
\end{figure}

\section{Discussion}
\label{sec:discussion}

\begin{table}[tb!]
    \centering
    \vspace{-3pt}
    \begin{minipage}[t]{0.49\textwidth}
        \centering
        \vspace{0pt}
        \captionof{table}{Effect of reward-model scale on EBD downstream task performance. Best values within each pre-trained model are bolded.}
        \label{tab:rm_scaling_perf}
        % \scriptsize
        \setlength{\tabcolsep}{5pt}
        \resizebox{\linewidth}{!}{%
        \begin{tabular}{l c c c c}
            \toprule
            RM Size & Math500 & GPQA & HumanEval & Writing \\
            \midrule
            \rowcolor{gray!15} \multicolumn{5}{l}{\textit{Pre-trained model: Qwen2.5-7B}} \\
            0.6B & 0.720 & \textbf{0.328} & \textbf{0.622} & \textbf{4.322} \\
            1.7B & 0.670 & 0.277 & 0.470 & 4.049 \\
            4B & \textbf{0.724} & 0.323 & 0.524 & 4.035 \\
            8B & 0.722 & 0.273 & 0.500 & 4.060 \\
            \midrule
            \rowcolor{gray!15} \multicolumn{5}{l}{\textit{Pre-trained model: Qwen3-8B-Base}} \\
            0.6B & 0.792 & \textbf{0.414} & \textbf{0.689} & \textbf{5.862} \\
            1.7B & 0.786 & 0.368 & 0.597 & 5.660 \\
            4B & 0.790 & 0.358 & 0.604 & 5.635 \\
            8B & \textbf{0.818} & 0.398 & 0.610 & 5.643 \\
            \midrule
            \rowcolor{gray!15} \multicolumn{5}{l}{\textit{Pre-trained model: Qwen2.5-14B}} \\
            0.6B & 0.724 & \textbf{0.368} & \textbf{0.561} & 3.676 \\
            1.7B & \textbf{0.732} & 0.358 & 0.518 & \textbf{4.236} \\
            4B & 0.722 & \textbf{0.368} & 0.537 & 3.727 \\
            8B & 0.724 & 0.353 & 0.549 & 4.206 \\
            \bottomrule
        \end{tabular}%
        }
    \end{minipage}%
    \hfill
    \begin{minipage}[t]{0.49\textwidth}
        \centering
        \vspace{0pt}
        \captionof{table}{Effect of reward-model scale on EBD inference latency (seconds). Lowest latency within each pre-trained model is bolded.}
        \label{tab:rm_scaling_eff}
        % \scriptsize
        \setlength{\tabcolsep}{5pt}
        \resizebox{\linewidth}{!}{%
        \begin{tabular}{l c c c c}
            \toprule
            RM Size & Math500 & GPQA & HumanEval & Writing \\
            \midrule
            \rowcolor{gray!15} \multicolumn{5}{l}{\textit{Pre-trained model: Qwen2.5-7B}} \\
            0.6B & 66.20 & 55.52 & 29.60 & \textbf{92.11} \\
            1.7B & 64.79 & \textbf{53.88} & 35.00 & 114.87 \\
            4B & 65.52 & 57.43 & 36.89 & 107.94 \\
            8B & \textbf{50.66} & 57.93 & \textbf{24.30} & 120.24 \\
            \midrule
            \rowcolor{gray!15} \multicolumn{5}{l}{\textit{Pre-trained model: Qwen3-8B-Base}} \\
            0.6B & 106.82 & \textbf{79.19} & 36.43 & 193.85 \\
            1.7B & 107.19 & 86.79 & 34.88 & 235.06 \\
            4B & \textbf{104.49} & 85.64 & \textbf{28.09} & 220.63 \\
            8B & 108.10 & 86.95 & 36.51 & \textbf{183.92} \\
            \midrule
            \rowcolor{gray!15} \multicolumn{5}{l}{\textit{Pre-trained model: Qwen2.5-14B}} \\
            0.6B & 95.86 & \textbf{71.83} & 81.94 & 177.43 \\
            1.7B & 97.32 & 75.35 & 71.53 & 143.15 \\
            4B & 97.90 & 72.88 & 60.67 & \textbf{139.13} \\
            8B & \textbf{94.64} & 74.63 & \textbf{60.18} & 143.24 \\
            \bottomrule
        \end{tabular}%
        }
    \end{minipage}
    \vspace{-10pt}
\end{table}

\noindent \textbf{Is reward-guided decoding robust to reward model scale?} 
Tables~\ref{tab:rm_scaling_perf} and \ref{tab:rm_scaling_eff} confirm performance stability across diverse tested configurations. Specifically, a 0.6B reward model nearly matches the 8B variant in Qwen2.5-7B Math500, achieving 0.720 versus 0.722 while reducing the video memory from 29.6 GB to 16.4 GB. These findings indicated that EBD did not depend on a specific large scale reward model and operated effectively when utilizing lightweight reward models.

\noindent \textbf{Does EBD improve performance through the same mechanism as post-training?}
We have established that EBD effectively activates instruction-following and reasoning capabilities in frozen pre-trained models. To probe whether this activation operates through the same mechanism as parameter updates, we evaluate EBD on models cold-start supervised fine-tuned on Dolly and Hybrid, as reported in Table~\ref{tab:discussion_coldstart}. On Math500, GPQA, and AlpacaEval, EBD performance scales positively with SFT. On HumanEval, however, EBD attains its highest score of $0.476$ at the no-SFT stage, and performance declines as SFT progresses (Dolly $0.459$, Hybrid $0.455$). Cold-start SFT for instruction following reshapes the pre-trained probability distribution, suppressing some continuation modes to establish new instruction-following behavior. For tasks like math, science reasoning, and open-ended generation, this redistribution is beneficial; for code completion, where the pre-trained distribution already encodes strong task-relevant behavior, the redistribution is detrimental. In contrast, EBD applies a reward tilt to the existing prior, not altering the underlying model distribution and without suppressing any modes. This explains the robustness of EBD. It operates on whatever distribution the model provides, tilting it toward higher utility without the mode-suppression side effect of SFT.
\begin{figure}[tb!]
  \centering
  \vspace{-2pt}
  \includegraphics[width=0.99\linewidth]{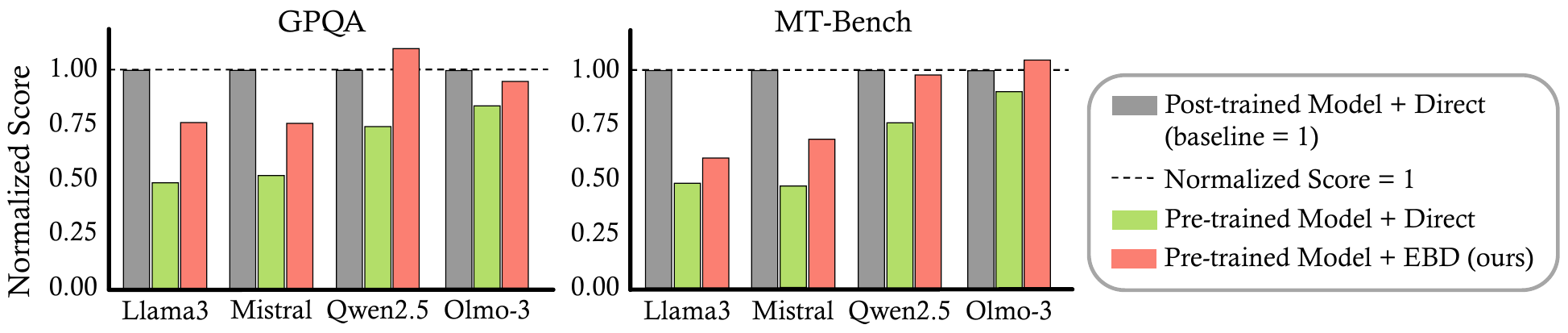}
  \vspace{-5pt}
  \caption{Improving pre-trained models without parameter updates. Direct decoding falls far below the 1.0 post-trained baseline, while EBD closes a substantial portion of this gap at inference time and in some cases matches or exceeds post-trained performance.}
  \label{fig:teaser}
  \vspace{-8pt}
\end{figure}

\begin{table}[tb!]
    \centering
    \vspace{-3pt}
    % 左边表格容器（加宽一点，让长内容更舒服）
    \begin{minipage}[t]{0.58\textwidth}
        \centering
        \captionof{table}{EBD performance across post-training stages on four benchmarks. EBD remains robust across training stages without altering the underlying model distribution.}
        \label{tab:discussion_coldstart}
        \footnotesize
        \setlength{\tabcolsep}{3pt}
        \begin{tabularx}{\linewidth}{@{} l *{4}{>{\centering\arraybackslash}X} @{}}
            \toprule
            Setting / method & Math500 & GPQA & HumanEval & Alpaca \\
            \midrule
            Direct (no SFT)          & 0.016 & 0.152 & 0.299 & 1.525 \\
            \textbf{EBD avg. (no SFT)}      & 0.030 & 0.173 & \textbf{0.476} & 1.625 \\
            \textbf{EBD avg. (Dolly)}    & 0.134 & 0.272 & 0.459 & 2.040 \\
            \textbf{EBD avg. (Hybrid)}  & \textbf{0.143} & \textbf{0.316} & 0.455 & \textbf{3.340} \\
            \bottomrule
        \end{tabularx}
    \end{minipage}%
    \hfill
    \begin{minipage}[t]{0.38\textwidth}
        \centering
        \captionof{table}{Hyperparameter search on validation set. Optimal values, accuracy, ranges for key hyperparameters.}
        \vspace{2pt}
        \label{tab:hyperparam_search}
        \footnotesize
        \setlength{\tabcolsep}{4pt}
        \begin{tabular}{@{} l c c c @{}}
            \toprule
            Study & Best value & Best acc. & Range \\
            \midrule
            beta          & \textbf{3.5} & 0.792 & 3--5 \\
            mcmc\_steps   & \textbf{12}  & 0.792 & 6--14 \\
            N\_candidates & \textbf{25}  & 0.814 & 1--25 \\
            block\_nums   & \textbf{24}  & 0.796 & 3--32 \\
            \bottomrule
        \end{tabular}
    \end{minipage}
    \vspace{-5pt}
\end{table}

\noindent \textbf{How does inference-time steering relate to post-trained models?}
Inference-time steering closes a substantial portion of the performance gap between pre-trained and post-trained models without parameter updates. As shown in Table~\ref{tab:main_results_full} and Figure~\ref{fig:teaser}, EBD applied to pre-trained models substantially narrows the normalized gap to post-trained direct-decoding baselines and occasionally matches or exceeds them. Figure~\ref{fig:teaser} visualizes this pattern across model families: Direct decoding consistently falls below the 1.0 post-trained reference, whereas EBD raises pre-trained scores to approach or even surpass that threshold. While post-training modifies parameters to improve task accessibility, EBD instead uses reward guidance to reach similar behavioral regions from a frozen state. The two approaches are complementary, allowing EBD to serve as an effective inference-time boost or a diagnostic for pre-trained checkpoint sensitivity.

\noindent \textbf{Is EBD stable across hyperparameter configurations?} 
EBD remains stable across its key hyperparameter ranges, as shown in Table~\ref{tab:hyperparam_search}. Inverse temperature $\beta$ is optimal at 3.5 with an accuracy of 0.792, but values between 3.0 and 4.5 show only minor variation. This indicates a forgiving calibration process without sharp performance drops. Accuracy plateaus at 12 MCMC steps and saturates as the block count reaches 12. These  results confirm that EBD does not require exhaustive per-task tuning to achieve strong effectiveness across a wide range of diverse benchmarks.

\section{Conclusion}
\label{sec:conclusion}

We propose Energy-Based Decoding (EBD), which is a training-free inference procedure that improves frozen pre-trained model outputs through energy-guided blockwise MCMC. Across five models and six benchmarks, EBD consistently outperforms direct decoding on both objective and subjective tasks while offering a substantially better quality and efficiency trade-off than Power Sampling. This procedure often reduces the inference latency by more than an order of magnitude. These gains are robust across reward model scales and families, persist in zero-shot and cold-start settings, and shift pre-trained model behavioral patterns toward post-trained signatures. This supports the view that decoding is a critical part of the evaluation problem itself. Standard decoding can underestimate frozen pre-trained models under task-oriented evaluation, and reward-guided inference improves alignment-relevant behavior without a single parameter update.

% \section*{References}
\bibliographystyle{plainnat}
\bibliography{references}

%%%%%%%%%%%%%%%%%%%%%%%%%%%%%%%%%%%%%%%%%%%%%%%%%%%%%%%%%%%%

%%%%%%%%%%%%%%%%%%%%%%%%%%%%%%%%%%%%%%%%%%%%%%%%%%%%%%%%%%%%
\clearpage
\appendix
\section{Case Studies}
\label{app:case_studies}
% \vspace{-50pt}
% We present 30 representative case studies that 
We present 30 representative case studies that illustrate how EBD activates task-oriented behavior from frozen pre-trained models. These examples span Math500 and GPQA across two base models, Llama3-8B and Qwen3-8B-Base, covering diverse domains including algebra, geometry, number theory, counting, precalculus, and graduate-level science reasoning. Each case contrasts Direct decoding with EBD and reveals a consistent pattern. Direct decoding frequently produces incoherent, socially-styled, or repetitive outputs that fail to deliver a parseable final answer, whereas EBD systematically generates structured, step-by-step solutions with well-formatted boxed conclusions, demonstrating that reward-guided refinement under the Metropolis--Hastings criterion can effectively steer frozen base models toward task-aligned generation without any parameter updates. 
\vspace{-100pt}
\begin{figure}[h]
    \centering
    % 第1张子图
    \begin{subfigure}[b]{1.0\textwidth}
        \centering
        \includegraphics[width=0.95\textwidth]{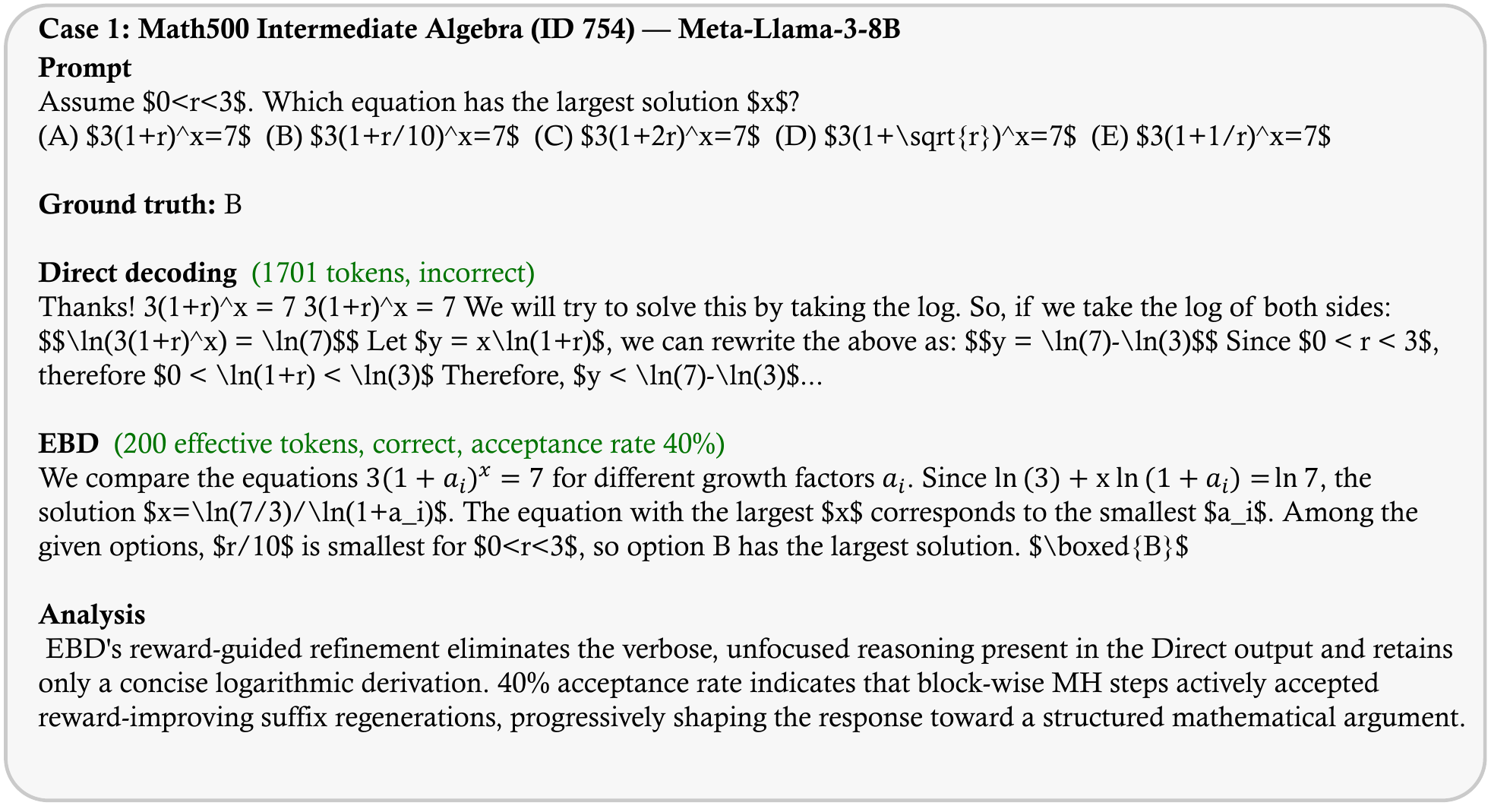}
        \label{fig:sub1}
    \end{subfigure}
    
    \vspace{20pt}
    
    % 第2张子图
    \begin{subfigure}[b]{1.0\textwidth}
        \centering
        \includegraphics[width=0.95\textwidth]{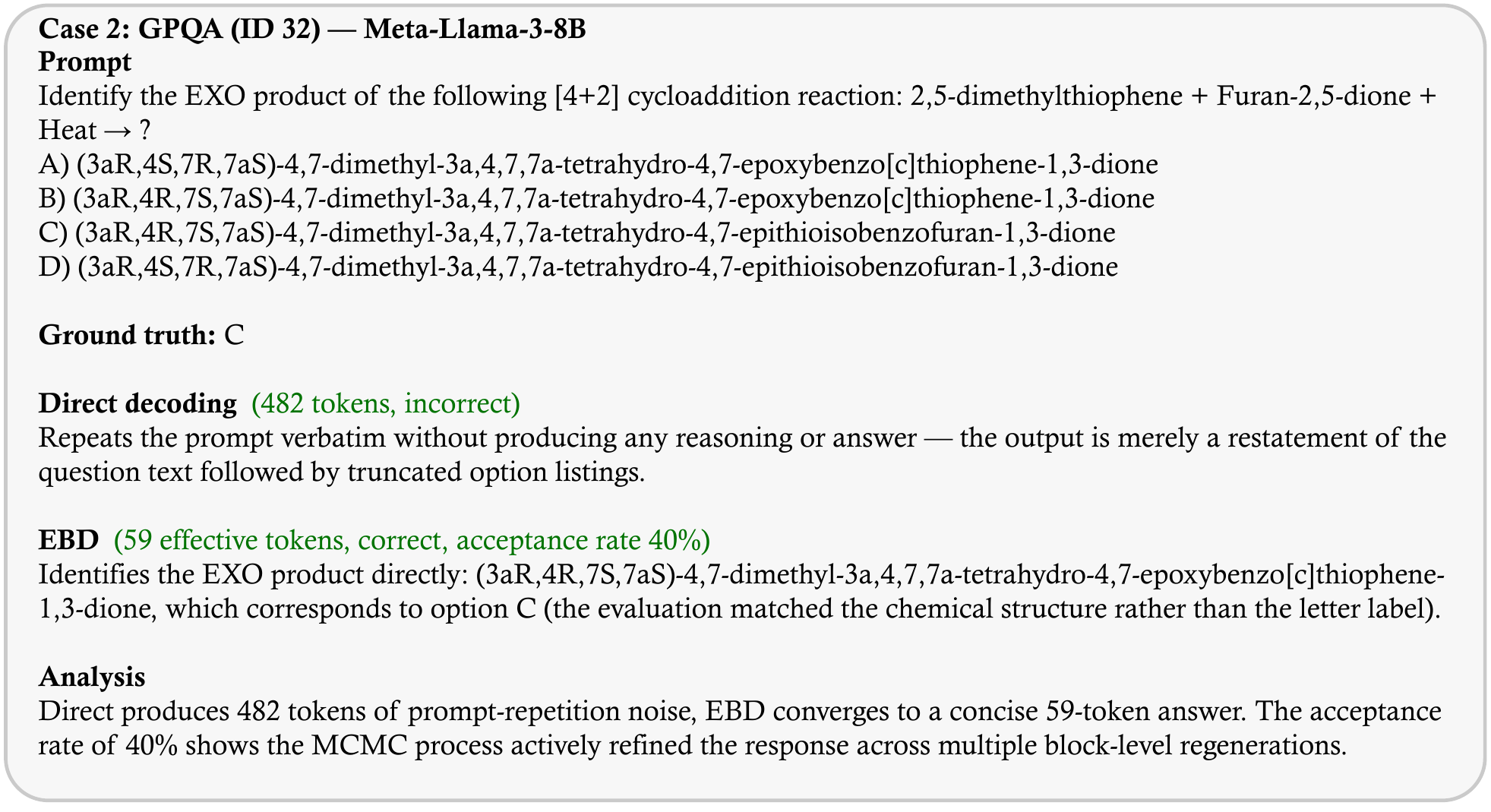}
    \end{subfigure}
    
    % \vspace{-5pt}
\end{figure}

\begin{figure}[H]
    \centering
    % 第3张子图
    \begin{subfigure}[b]{1.0\textwidth}
        \centering
        \includegraphics[width=0.95\textwidth]{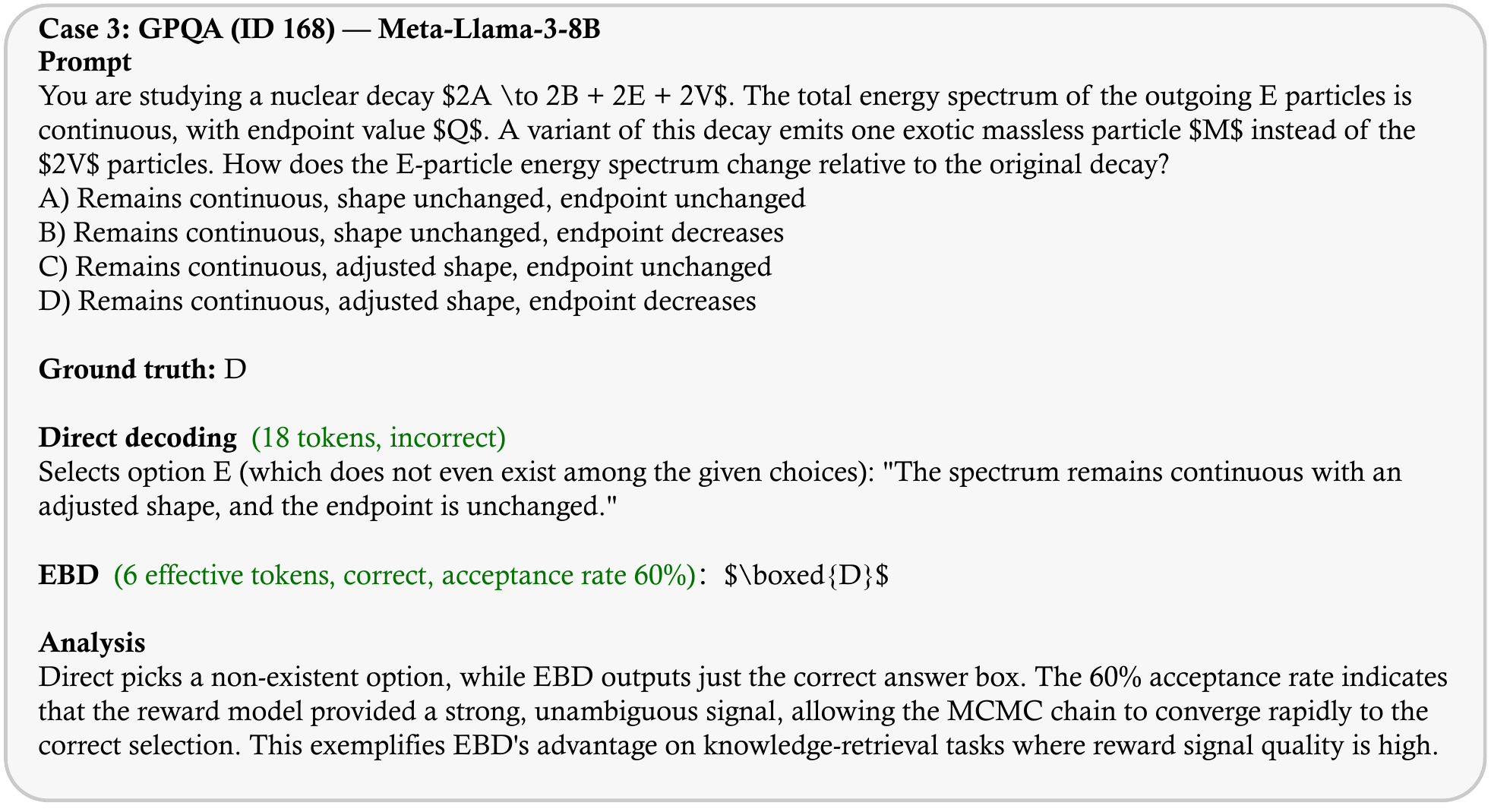}
        \label{fig:sub3}
    \end{subfigure}
    
    % 第1张子图
    \begin{subfigure}[b]{1.0\textwidth}
        \centering
        \includegraphics[width=0.95\textwidth]{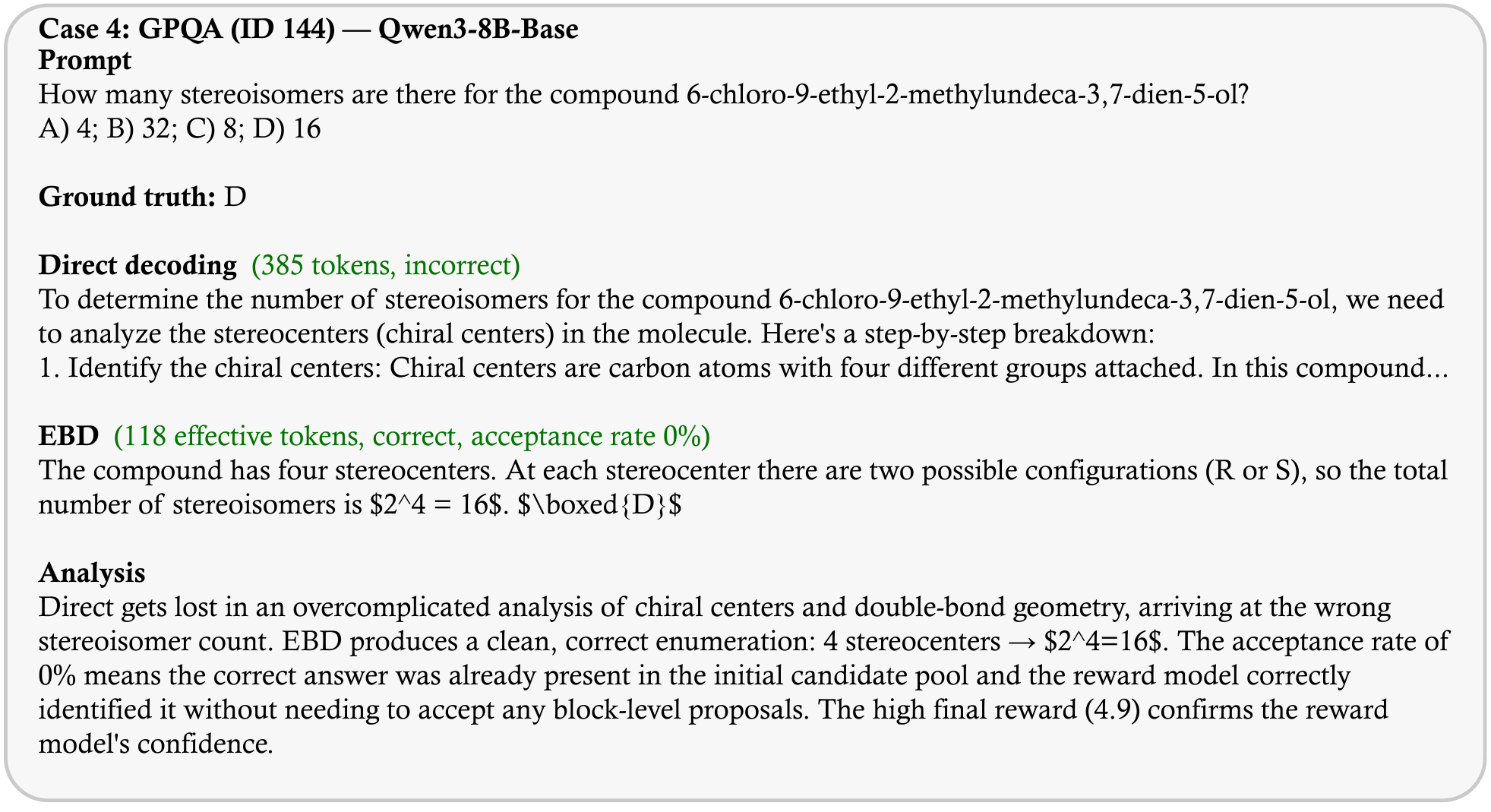}
        \label{fig:sub1}
    \end{subfigure}
    
    % \vspace{-12pt}
    
    % 第2张子图
    \begin{subfigure}[b]{1.0\textwidth}
        \centering
        \includegraphics[width=0.95\textwidth]{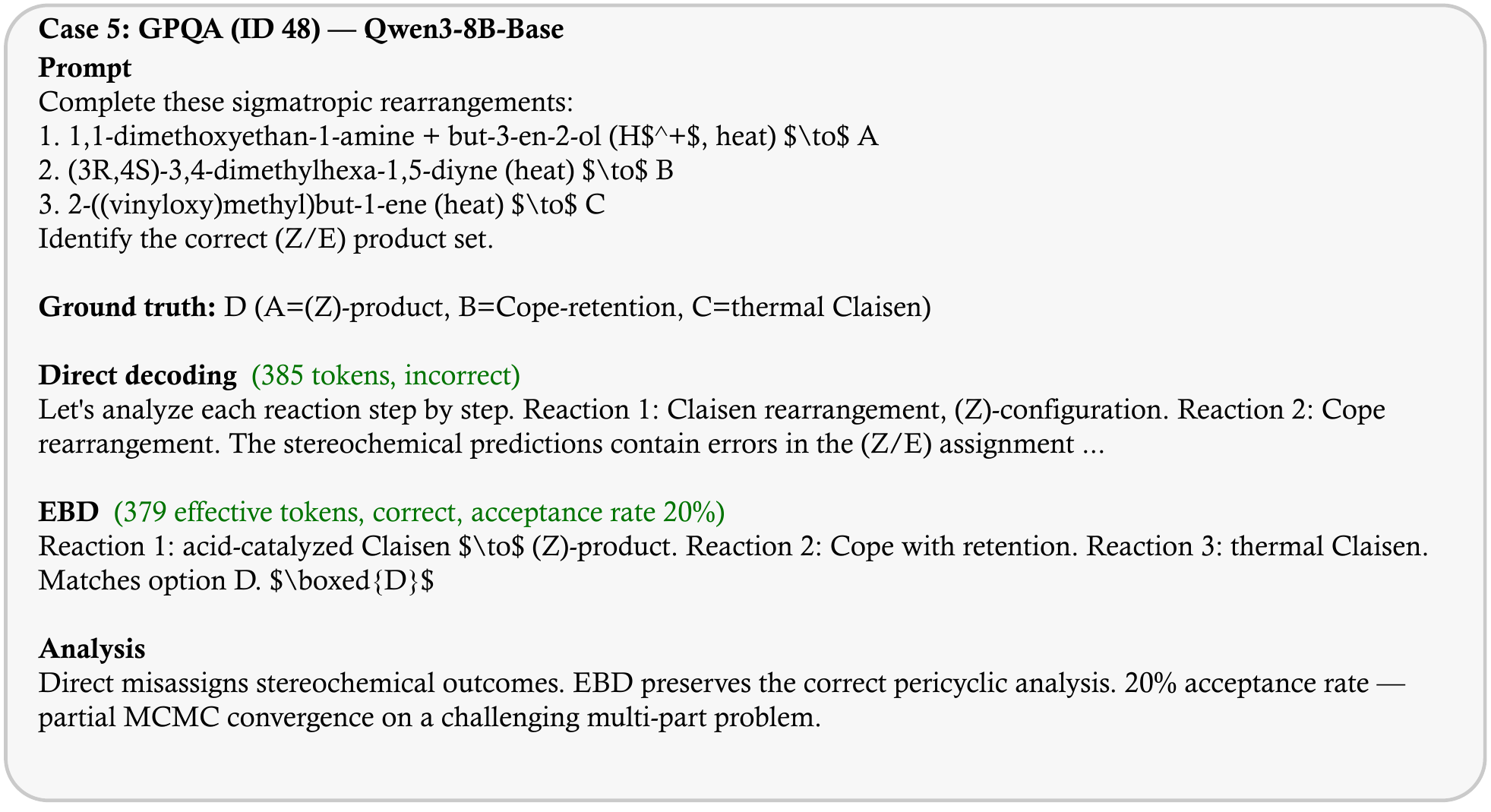}
    \end{subfigure}
    
    % \vspace{-5pt}

\end{figure}

\begin{figure}[H]
    \centering

        % 第3张子图
    \begin{subfigure}[b]{1.0\textwidth}
        \centering
        \includegraphics[width=0.95\textwidth]{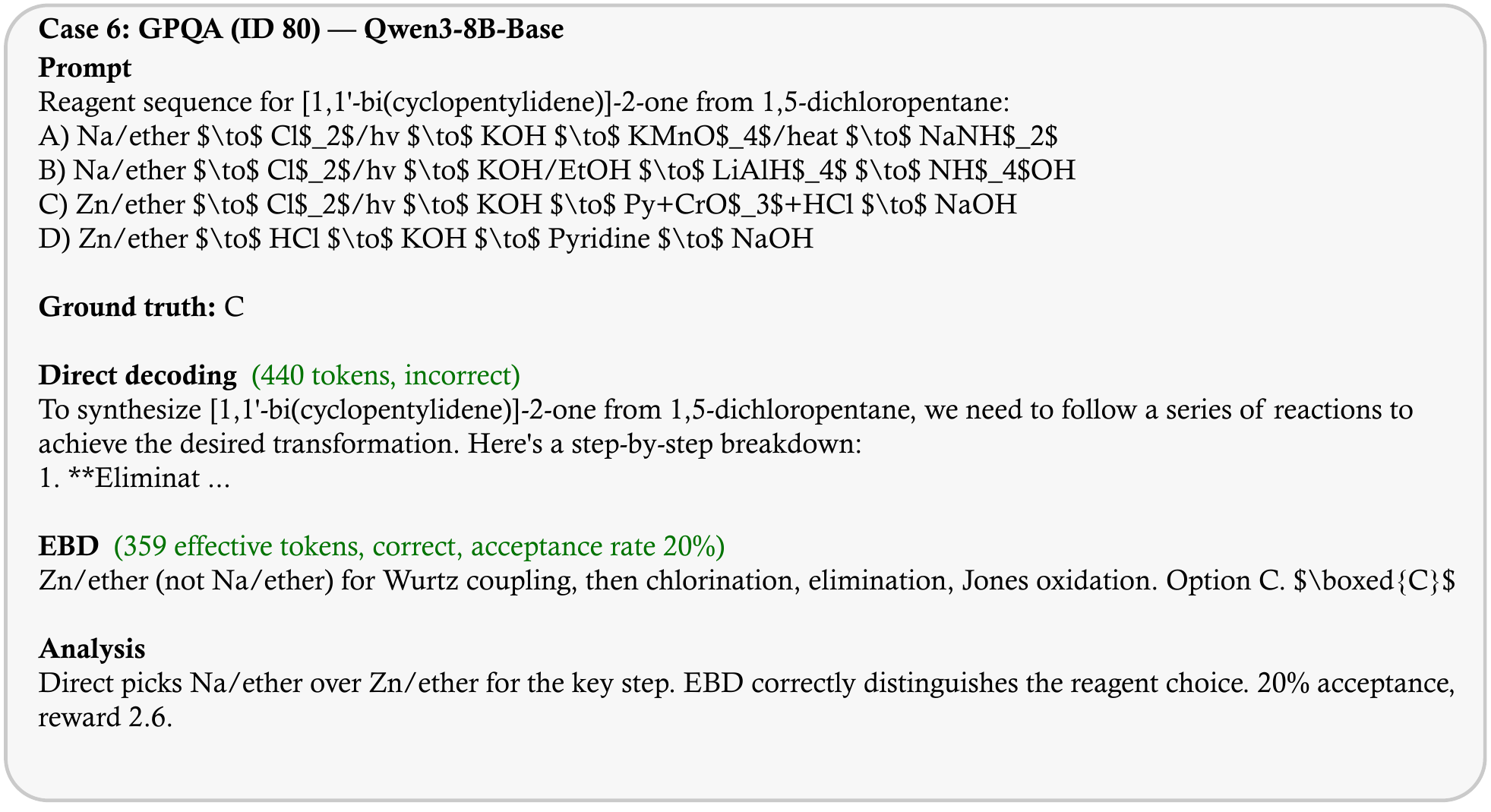}
        \label{fig:sub3}
    \end{subfigure}
    
    % 第1张子图
    \begin{subfigure}[b]{1.0\textwidth}
        \centering
        \includegraphics[width=0.95\textwidth]{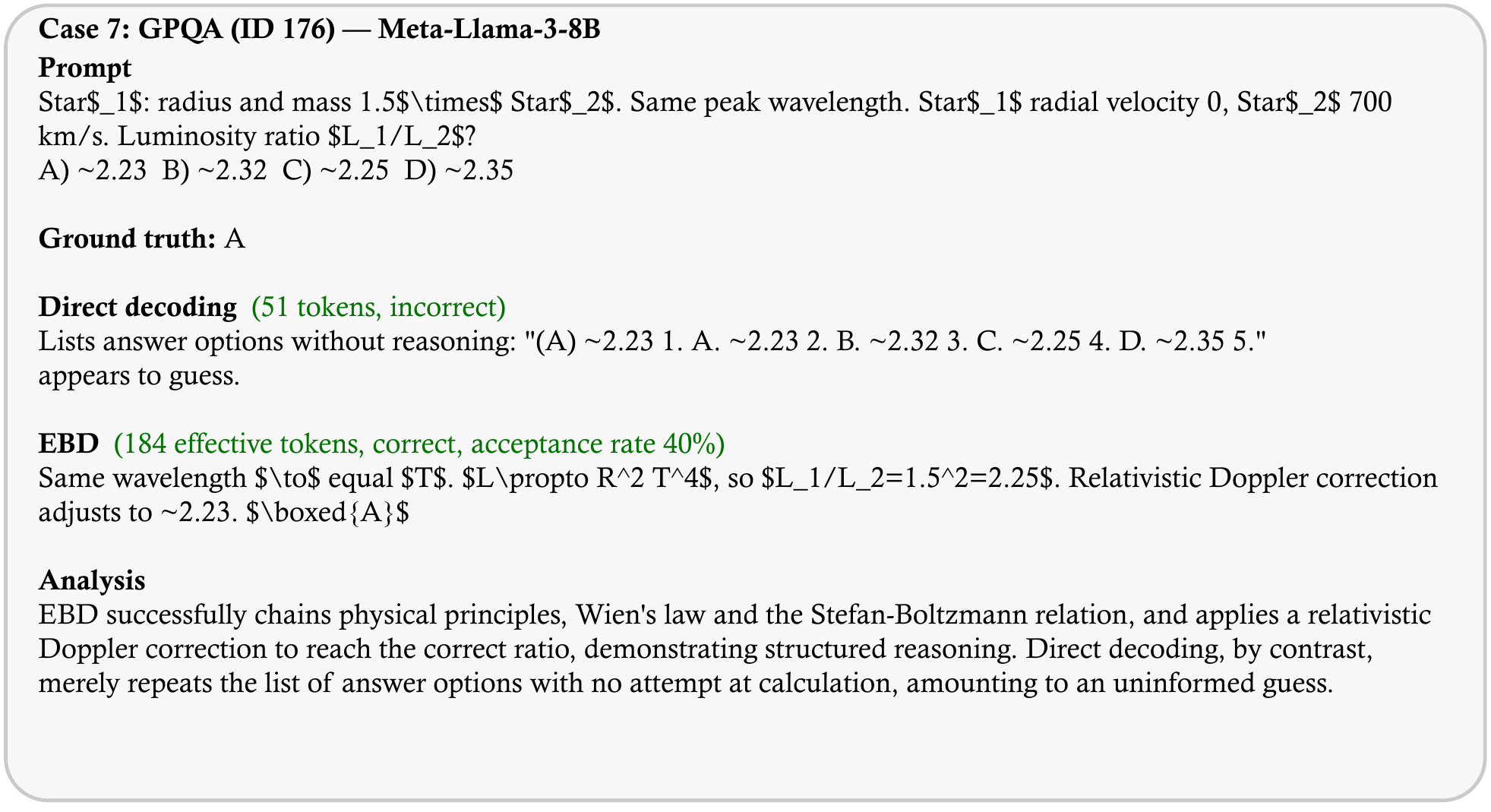}
        \label{fig:sub1}
    \end{subfigure}
    
    % \vspace{-12pt}
    
    % 第2张子图
    \begin{subfigure}[b]{1.0\textwidth}
        \centering
        \includegraphics[width=0.95\textwidth]{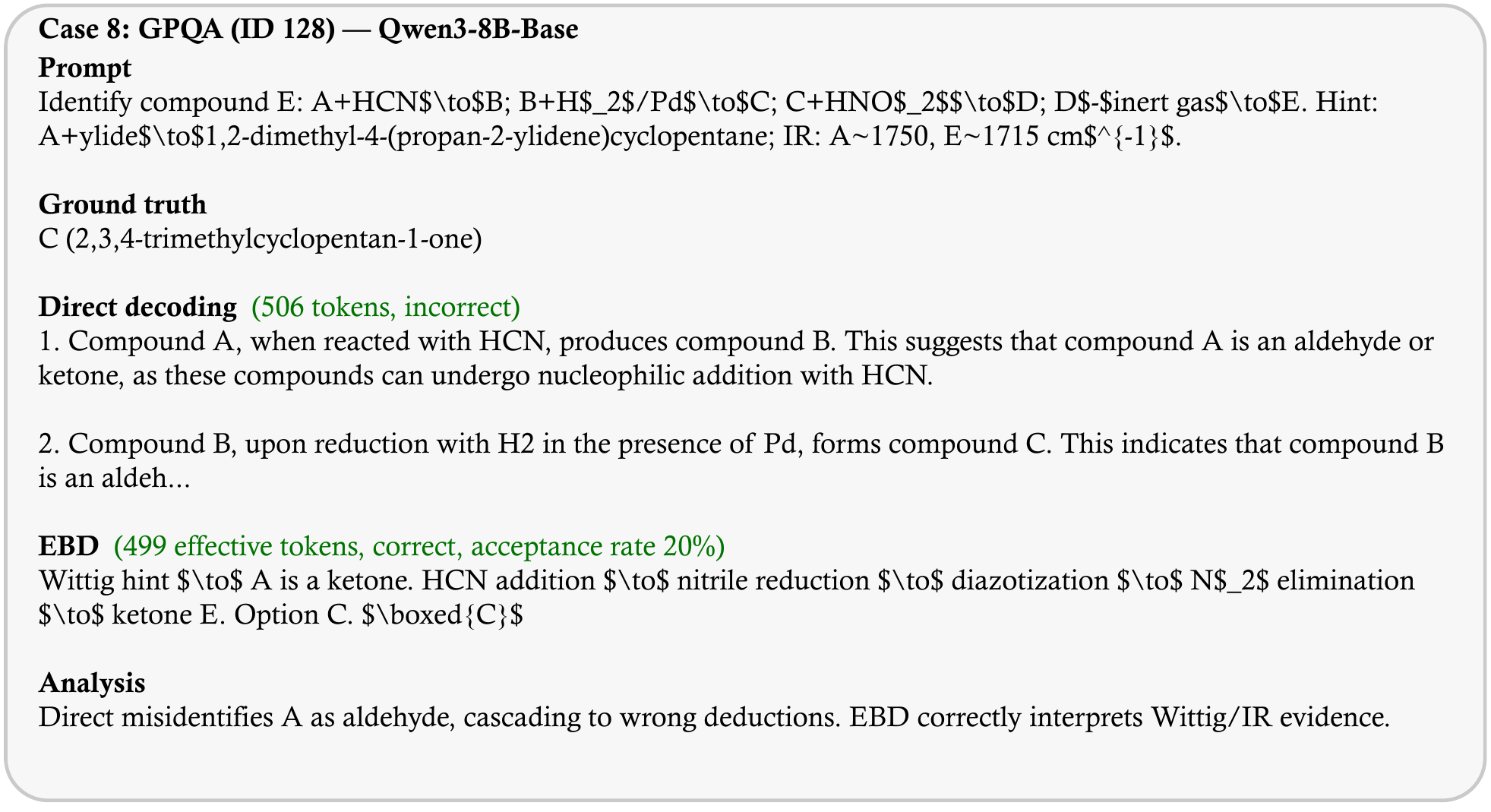}
    \end{subfigure}
    
    % \vspace{-5pt}
\end{figure}

\begin{figure}[H]
    \centering

        % 第3张子图
    \begin{subfigure}[b]{1.0\textwidth}
        \centering
        \includegraphics[width=0.95\textwidth]{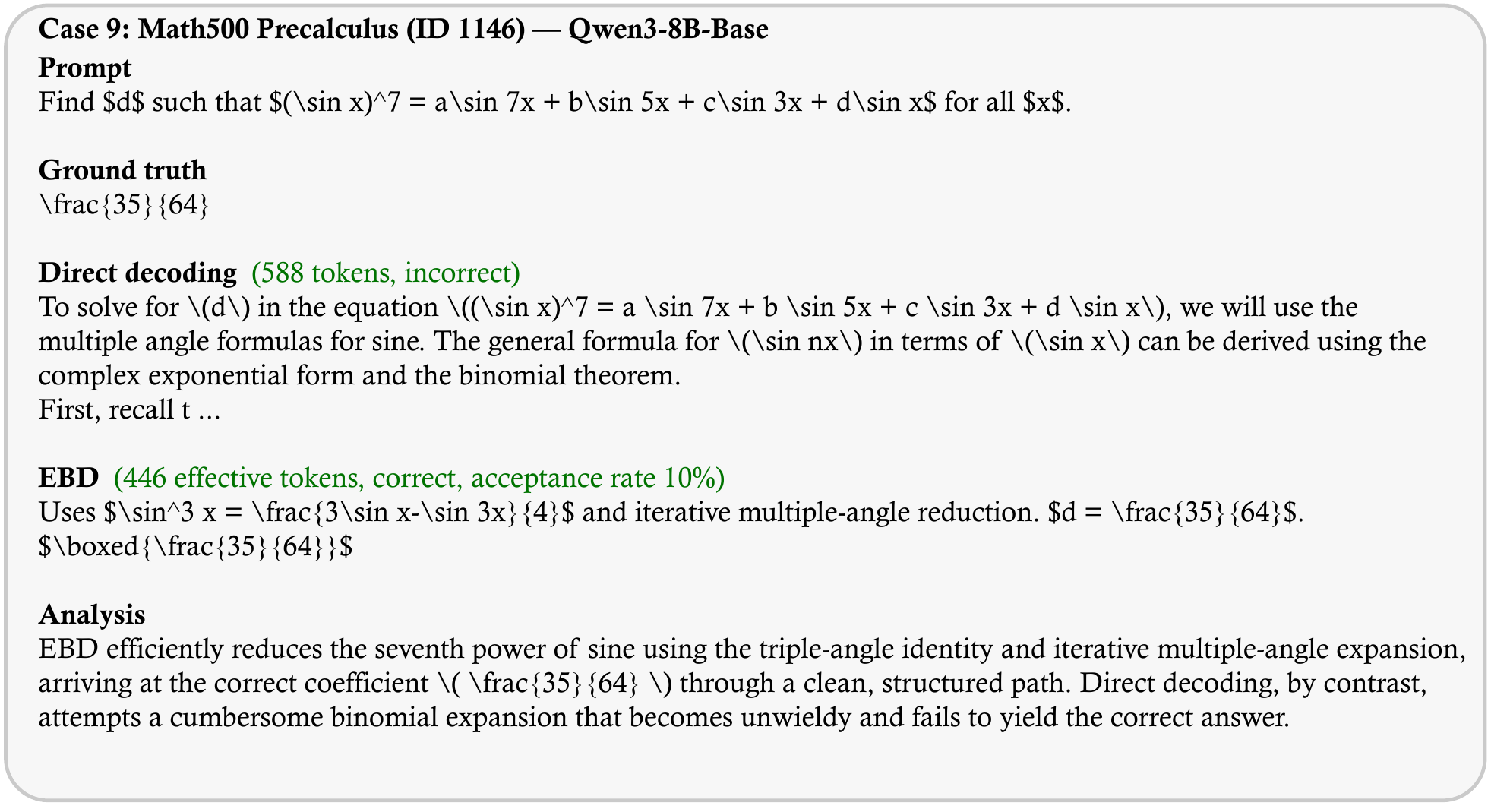}
        \label{fig:sub3}
    \end{subfigure}
    
    % 第1张子图
    \begin{subfigure}[b]{1.0\textwidth}
        \centering
        \includegraphics[width=0.95\textwidth]{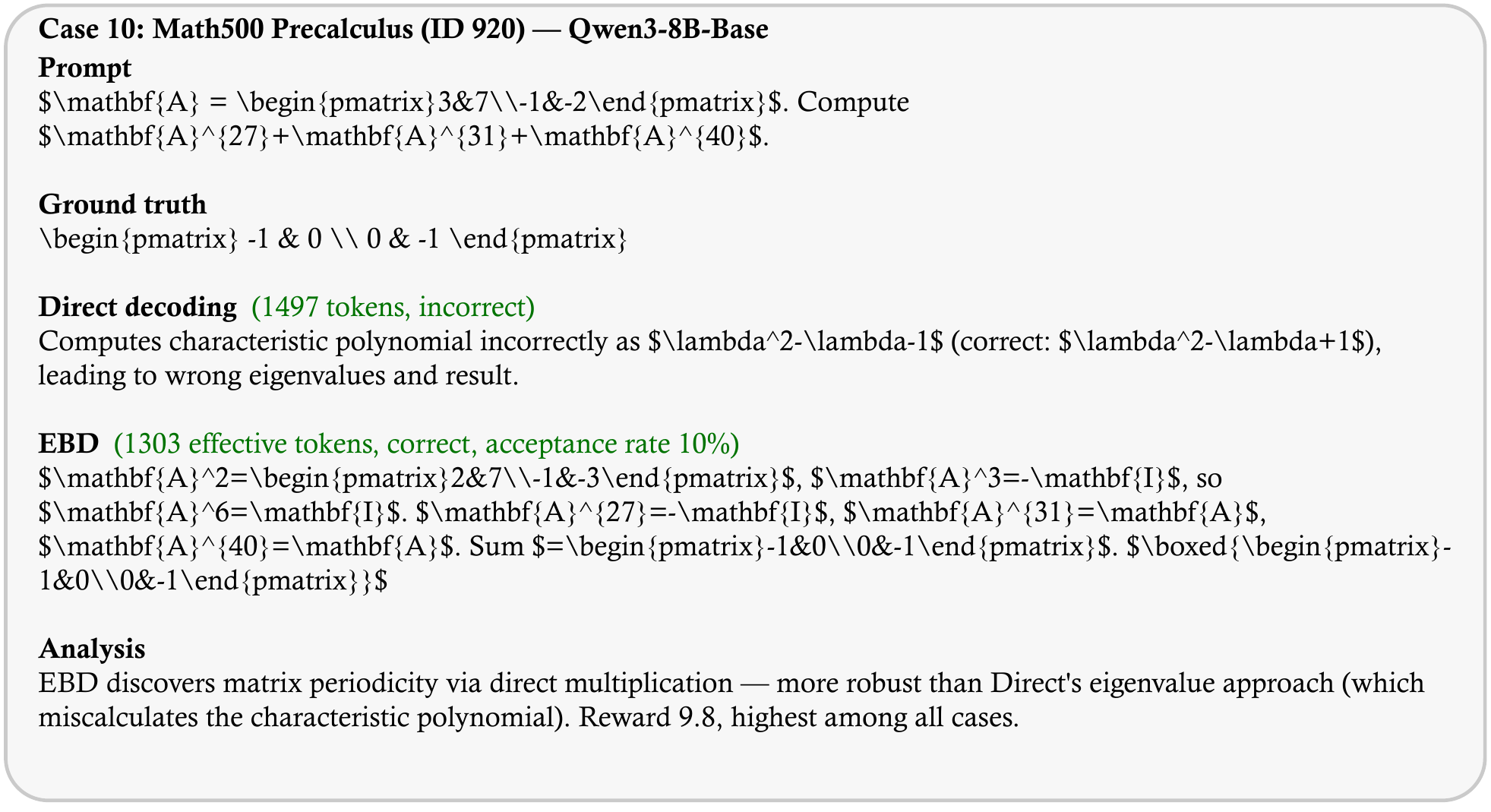}
        \label{fig:sub1}
    \end{subfigure}
    
    % \vspace{-12pt}
    
    % 第2张子图
    \begin{subfigure}[b]{1.0\textwidth}
        \centering
        \includegraphics[width=0.95\textwidth]{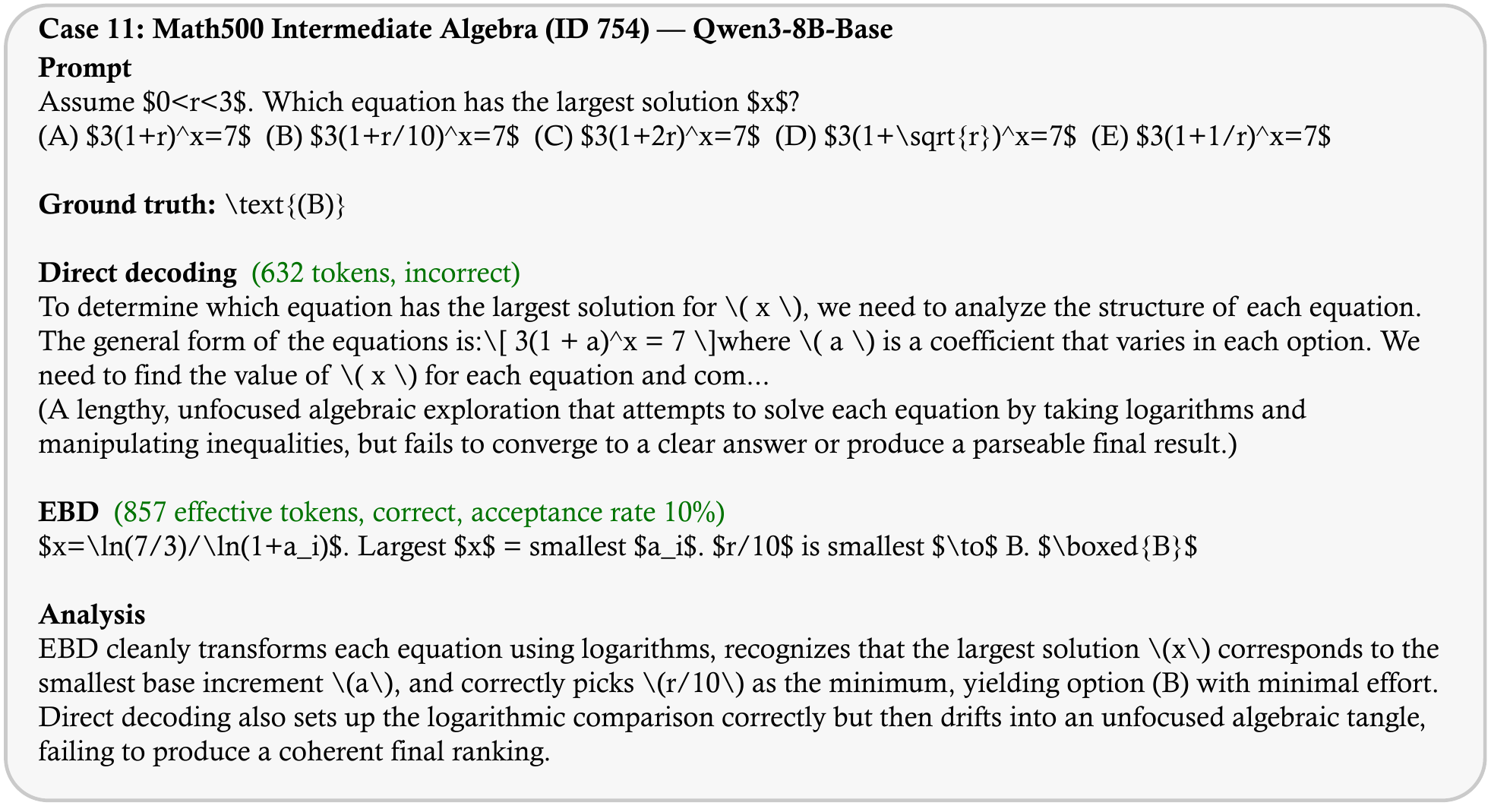}
    \end{subfigure}
    
    % \vspace{-5pt}
    
\end{figure}

\begin{figure}[H]
    \centering
        % 第3张子图
    \begin{subfigure}[b]{1.0\textwidth}
        \centering
        \includegraphics[width=0.95\textwidth]{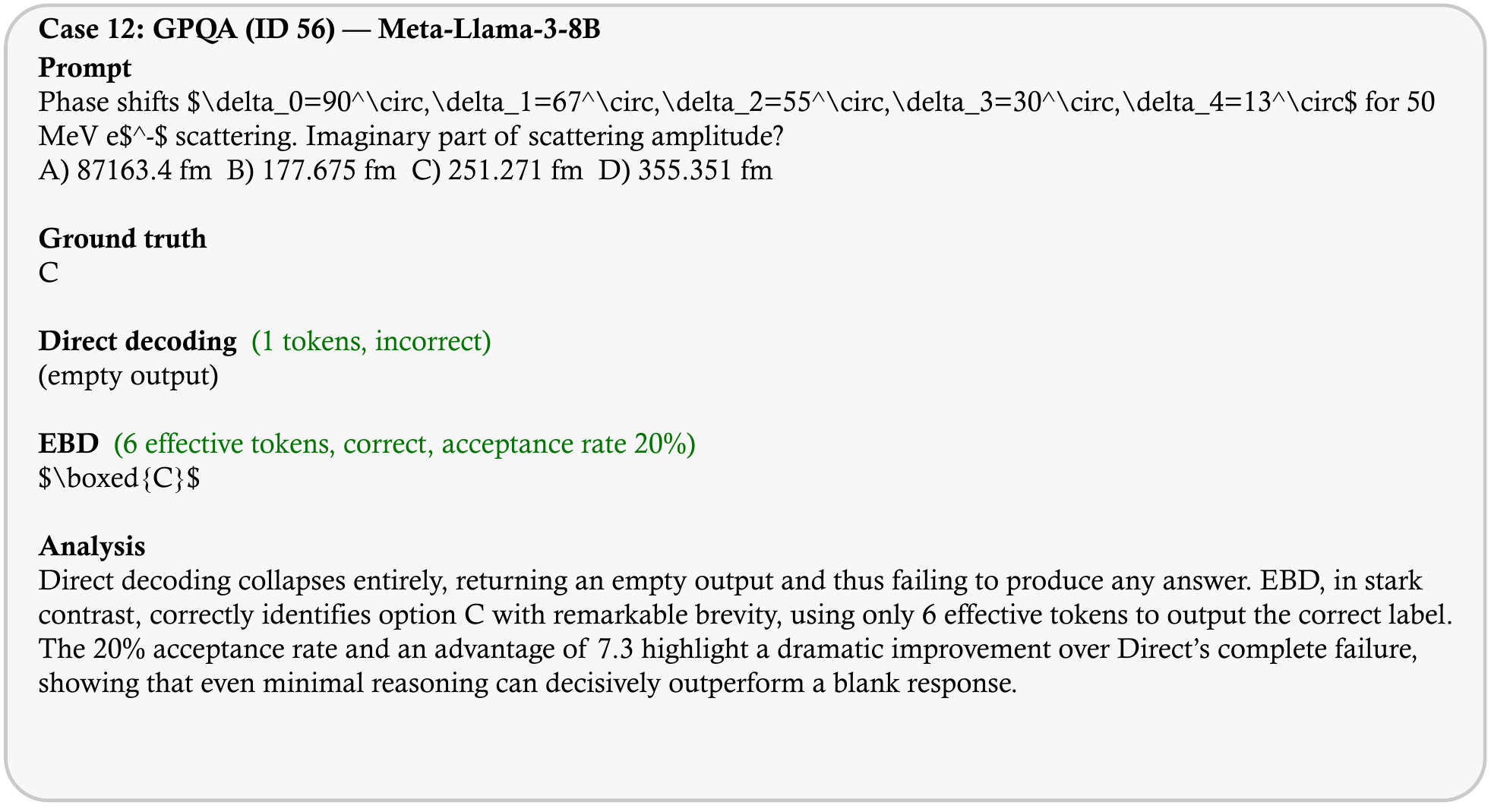}
        \label{fig:sub3}
    \end{subfigure}
    
    % 第1张子图
    \begin{subfigure}[b]{1.0\textwidth}
        \centering
        \includegraphics[width=0.95\textwidth]{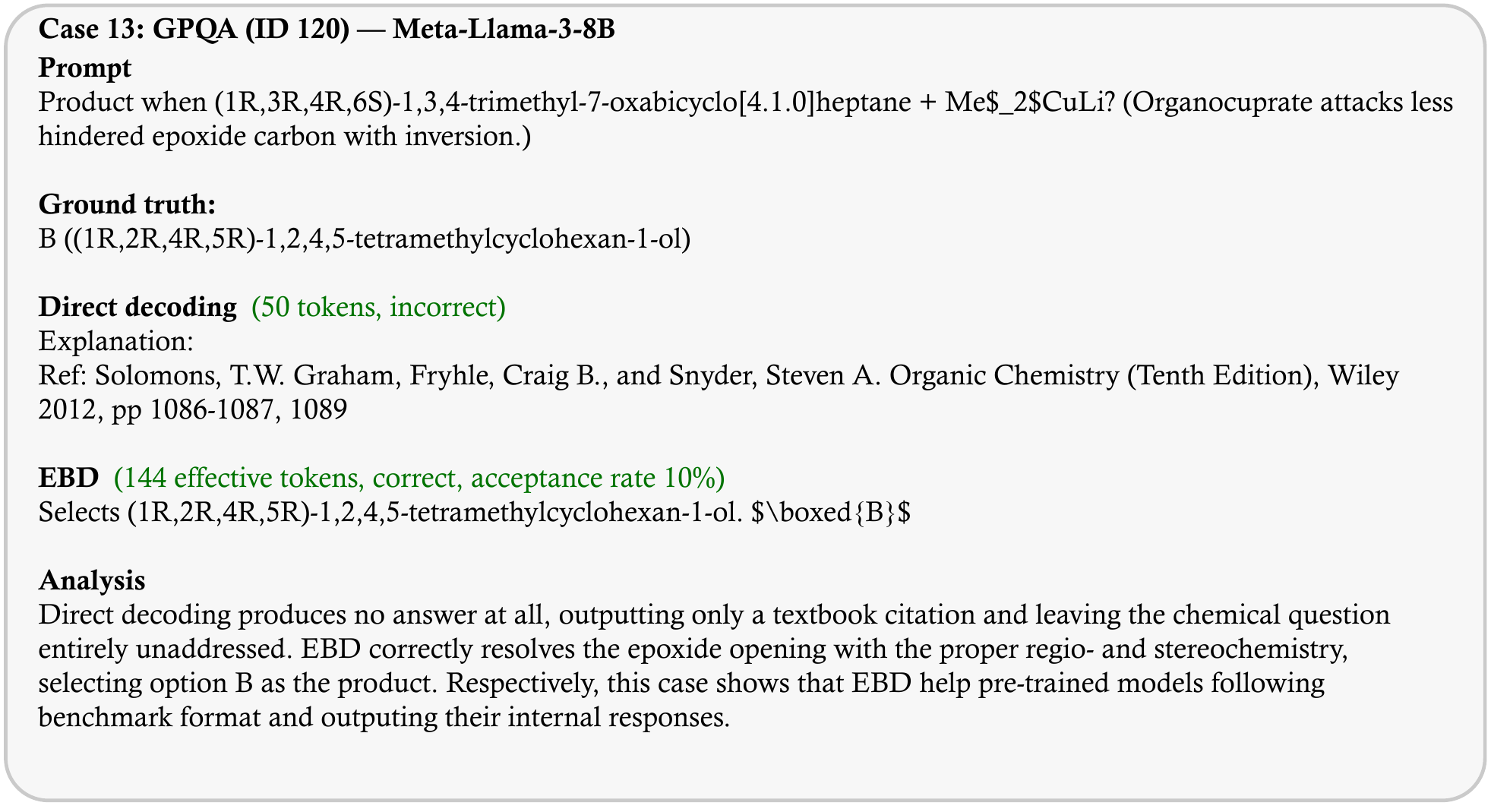}
        \label{fig:sub1}
    \end{subfigure}
    
    % \vspace{-12pt}
    
    % 第2张子图
    \begin{subfigure}[b]{1.0\textwidth}
        \centering
        \includegraphics[width=0.95\textwidth]{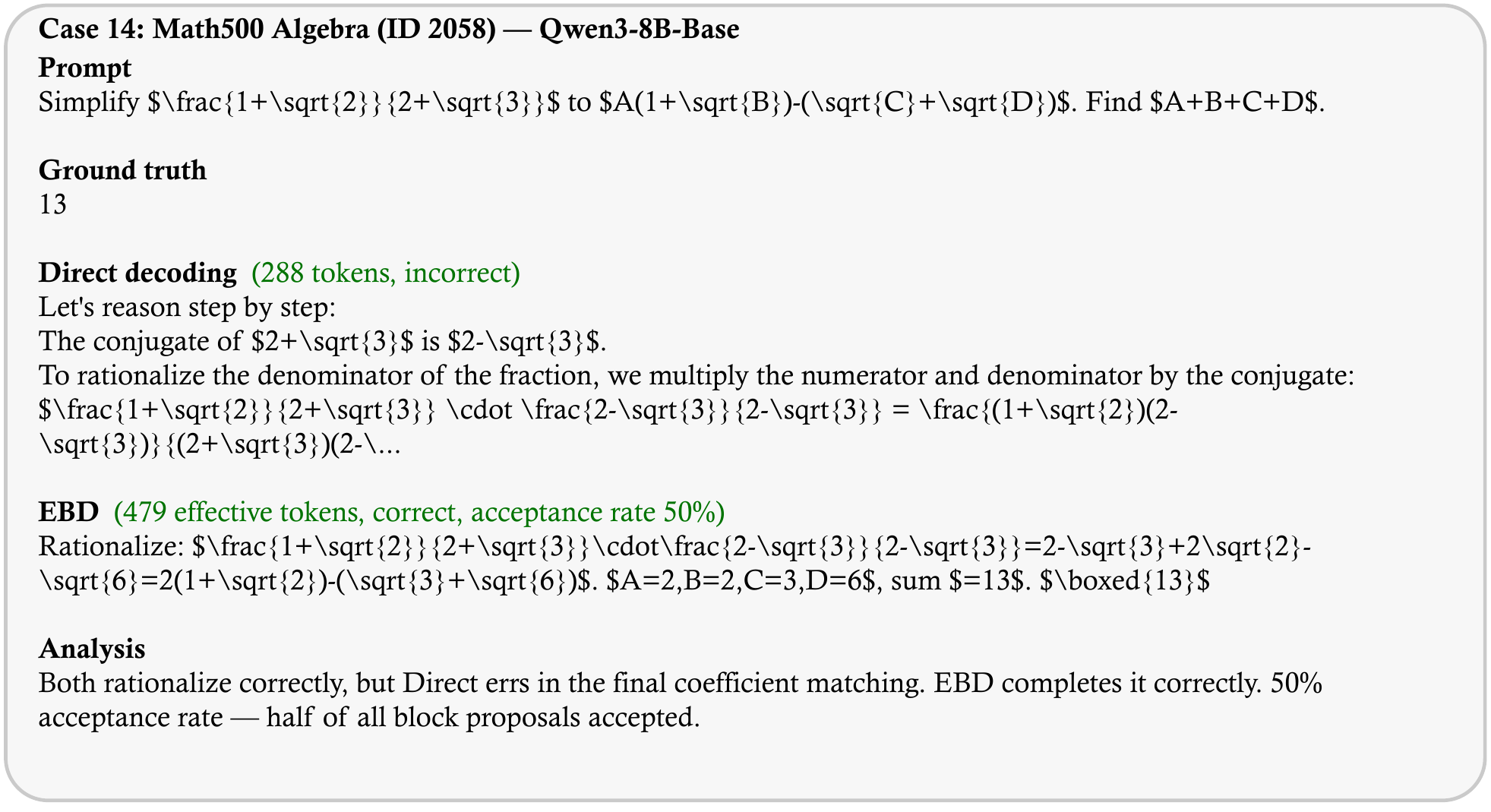}
    \end{subfigure}
    
    % \vspace{-5pt}

\end{figure}

\begin{figure}[H]
    \centering

        % 第3张子图
    \begin{subfigure}[b]{1.0\textwidth}
        \centering
        \includegraphics[width=0.95\textwidth]{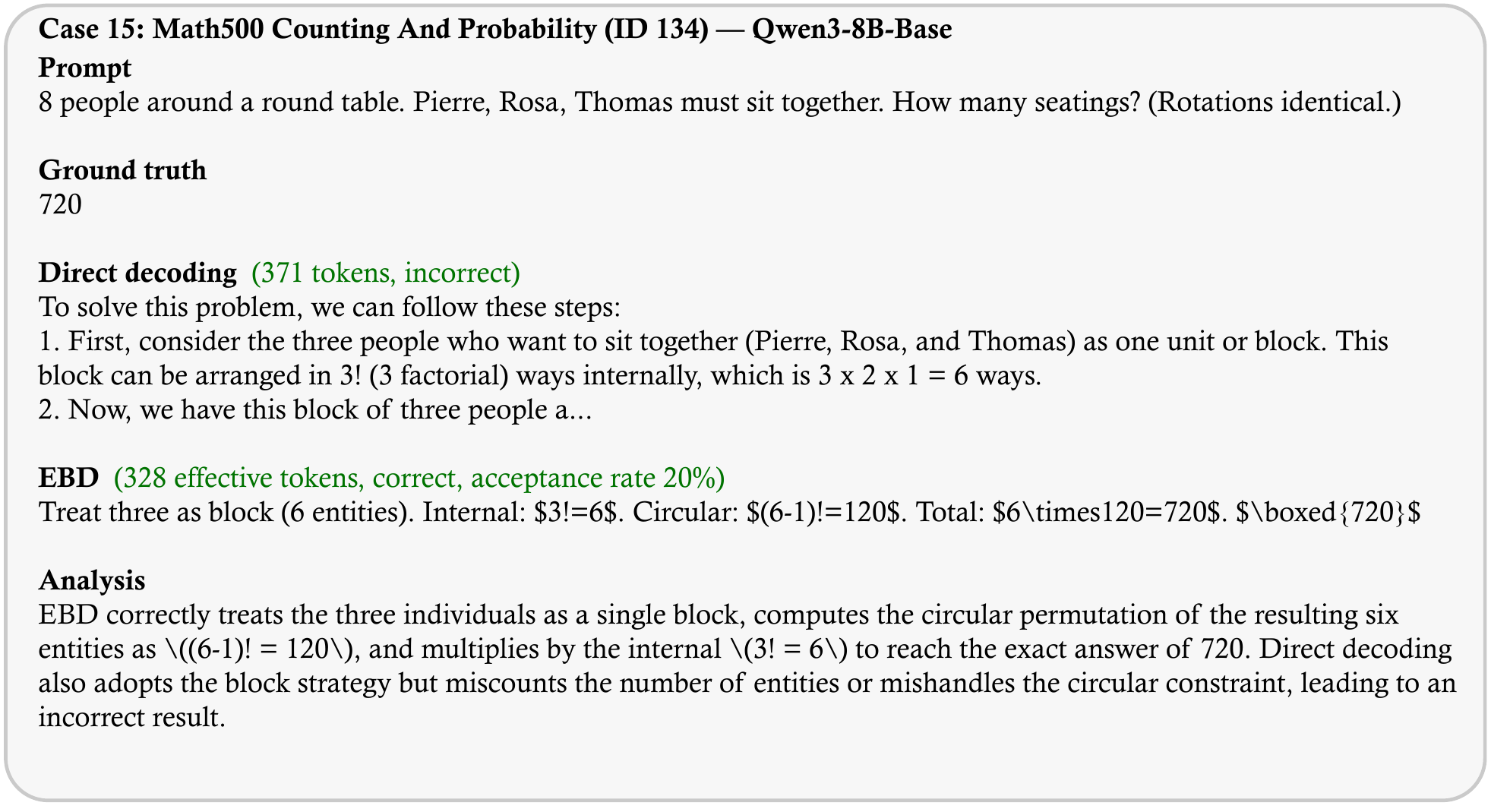}
        \label{fig:sub3}
    \end{subfigure}
    
    % 第1张子图
    \begin{subfigure}[b]{1.0\textwidth}
        \centering
        \includegraphics[width=0.95\textwidth]{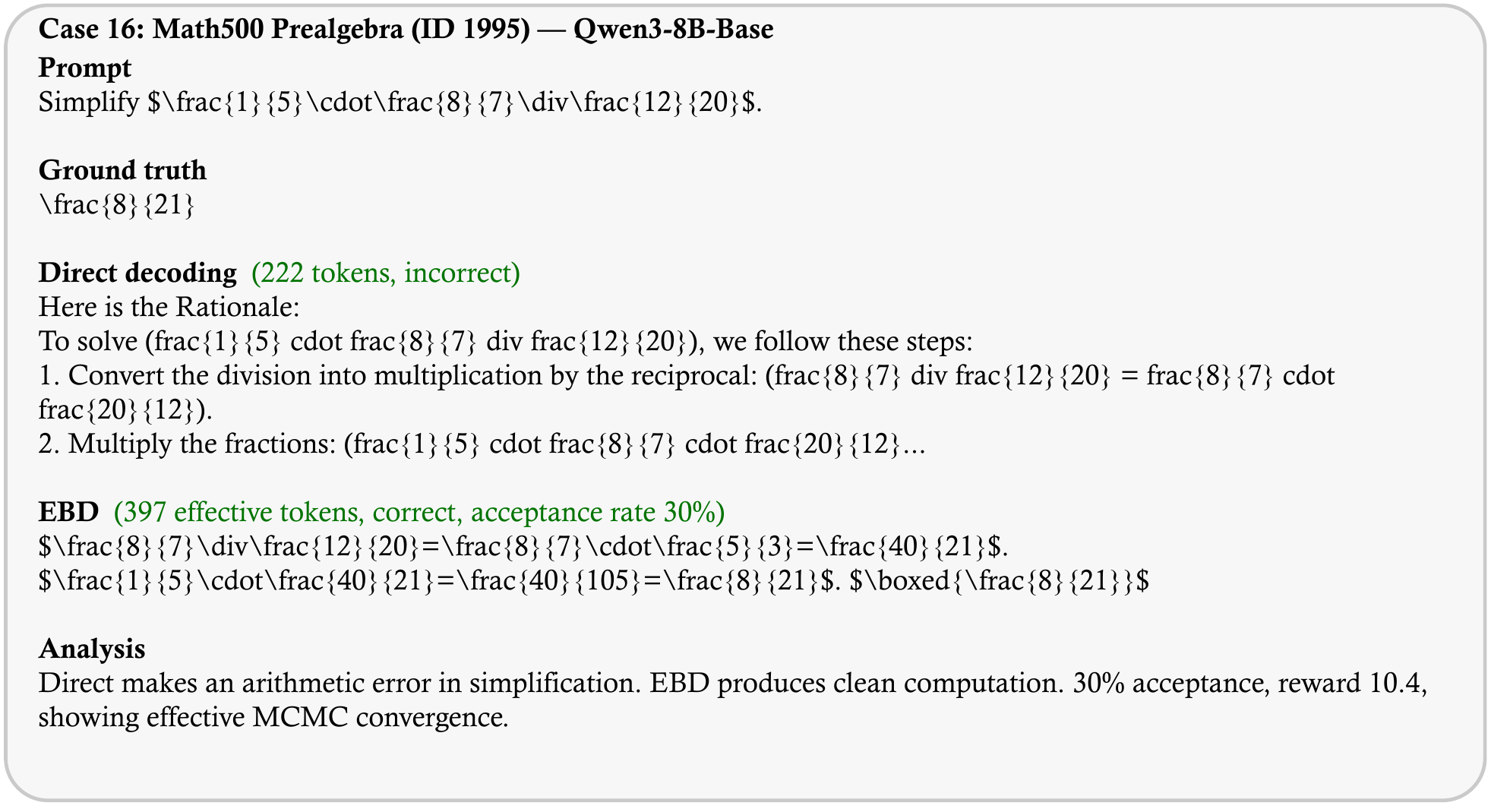}
        \label{fig:sub1}
    \end{subfigure}
    
    % \vspace{-12pt}
    
    % 第2张子图
    \begin{subfigure}[b]{1.0\textwidth}
        \centering
        \includegraphics[width=0.95\textwidth]{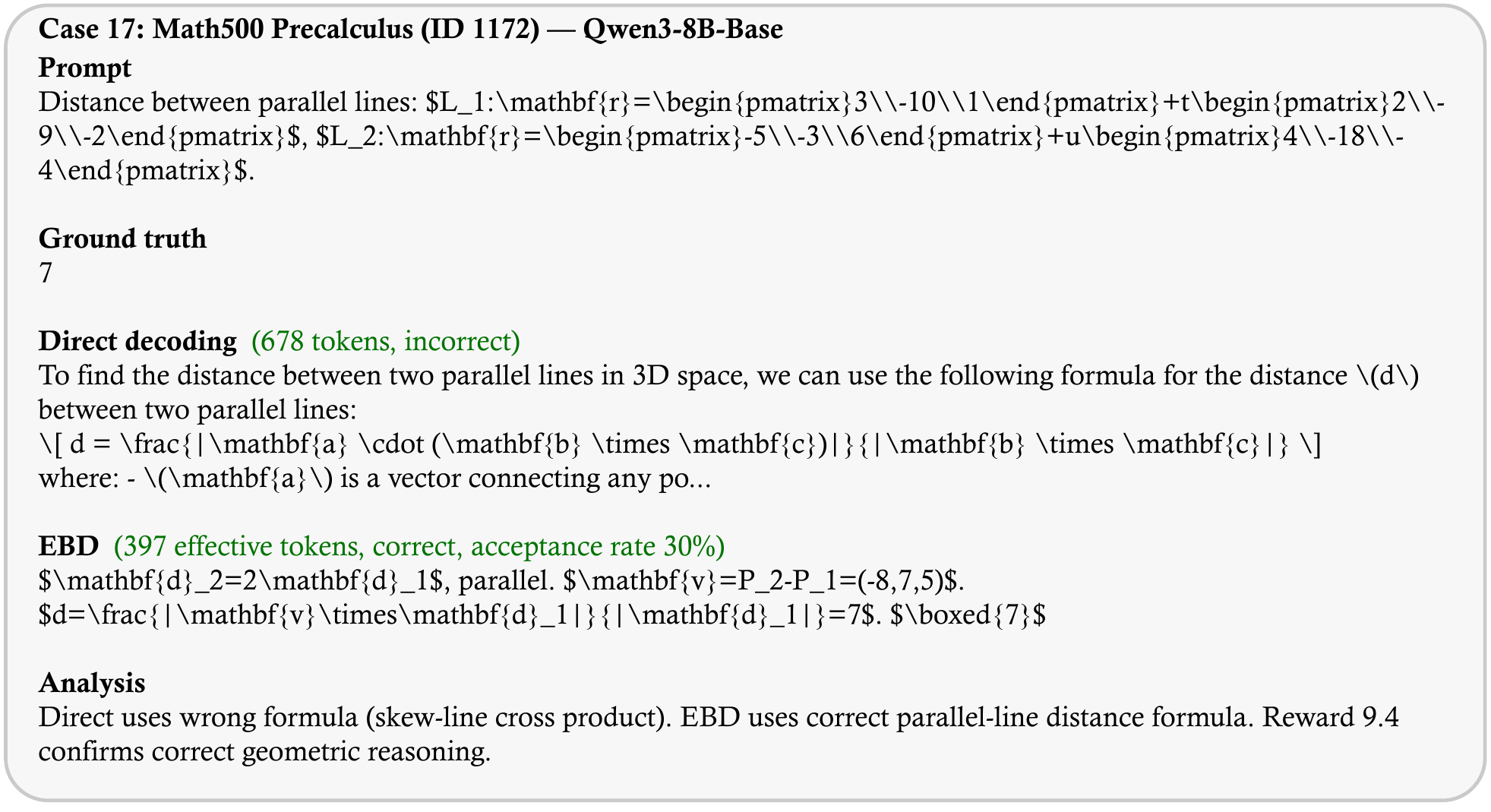}
    \end{subfigure}
    
    % \vspace{-5pt}
\end{figure}

\begin{figure}[H]
    \centering

        % 第3张子图
    \begin{subfigure}[b]{1.0\textwidth}
        \centering
        \includegraphics[width=0.95\textwidth]{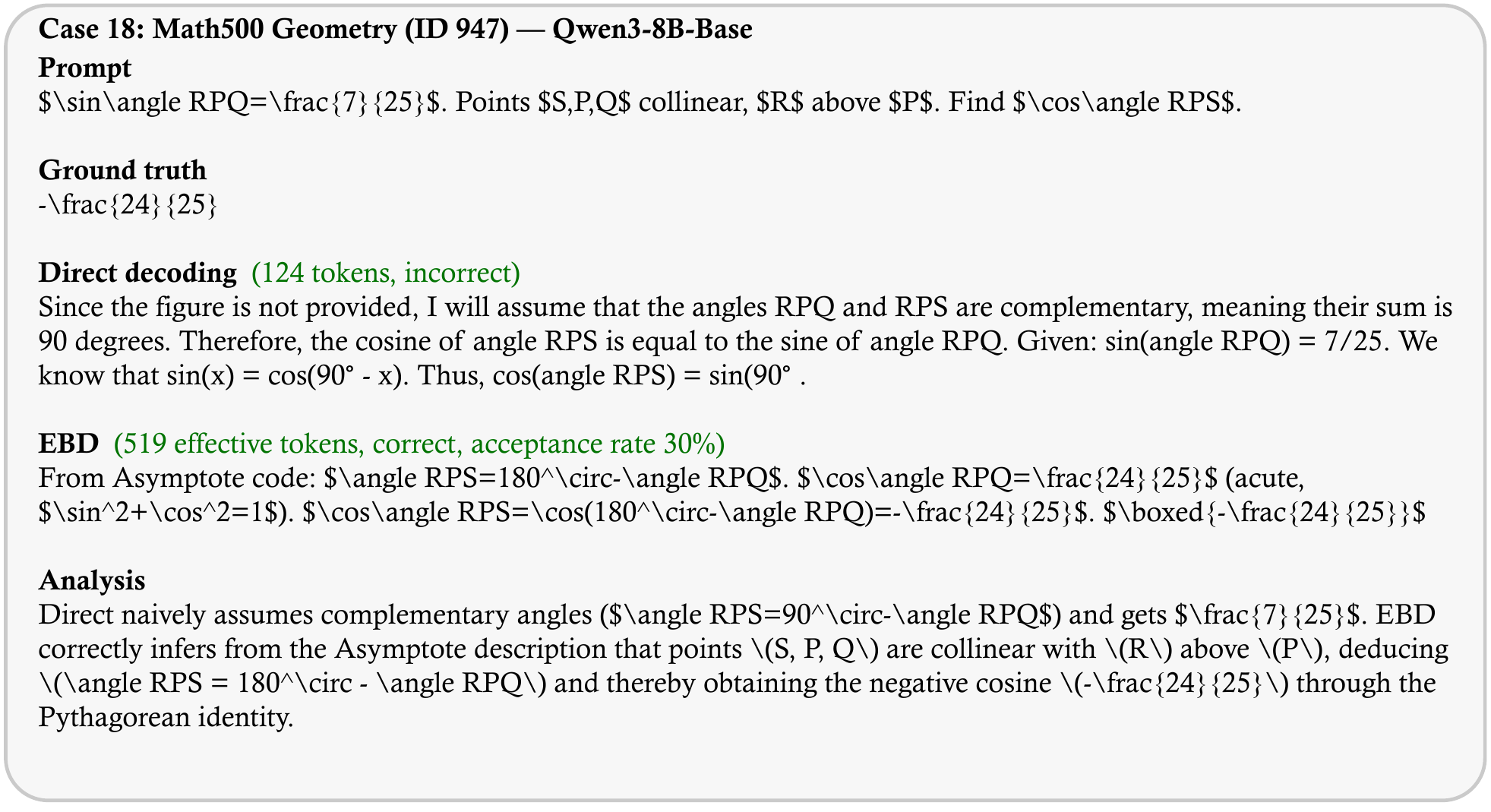}
        \label{fig:sub3}
    \end{subfigure}
    
    % 第1张子图
    \begin{subfigure}[b]{1.0\textwidth}
        \centering
        \includegraphics[width=0.95\textwidth]{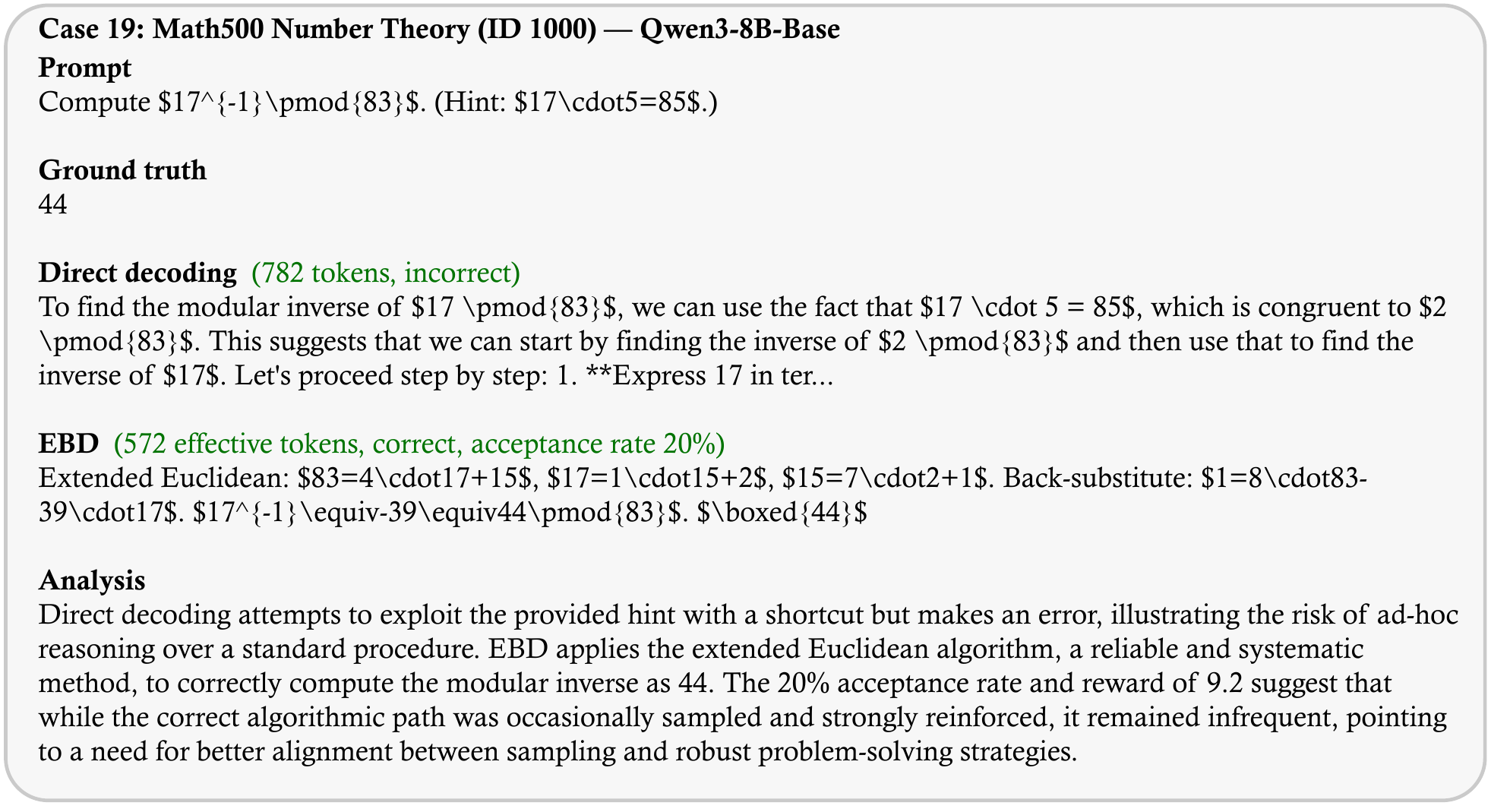}
        \label{fig:sub1}
    \end{subfigure}
    
    % \vspace{-12pt}
    
    % 第2张子图
    \begin{subfigure}[b]{1.0\textwidth}
        \centering
        \includegraphics[width=0.95\textwidth]{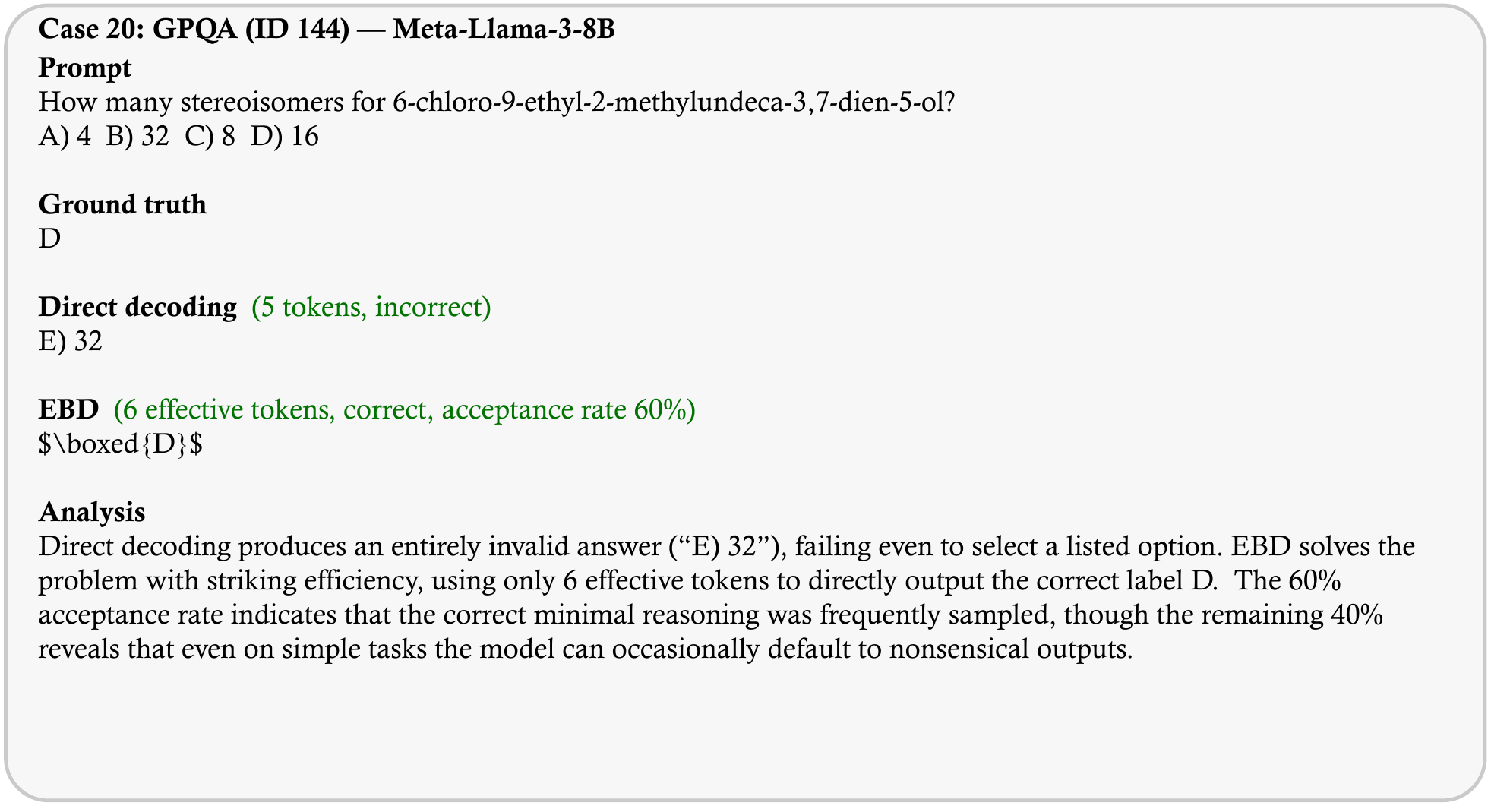}
    \end{subfigure}
    
    % \vspace{-5pt}

\end{figure}

\begin{figure}[H]
    \centering

        % 第3张子图
    \begin{subfigure}[b]{1.0\textwidth}
        \centering
        \includegraphics[width=0.95\textwidth]{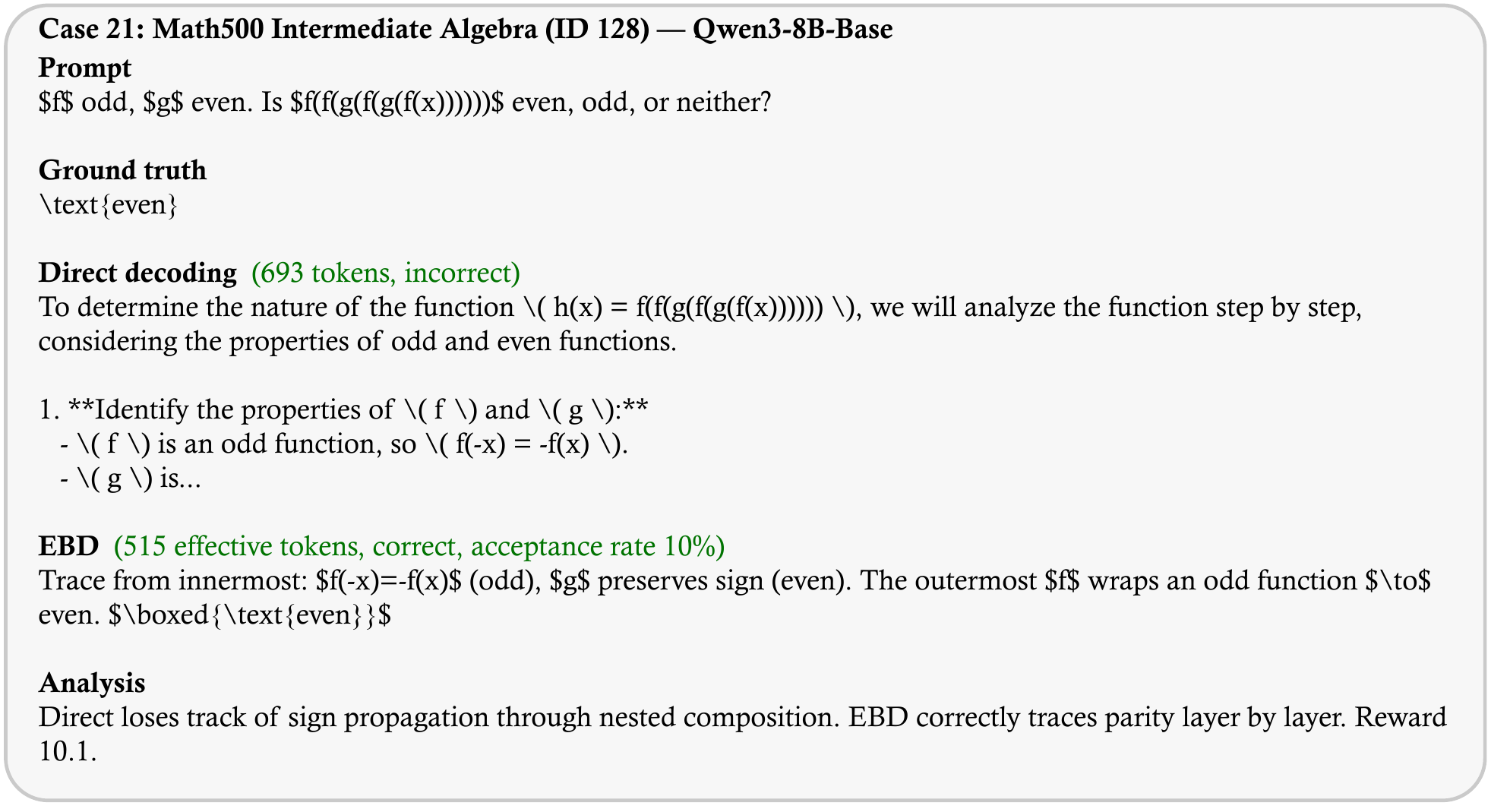}
        \label{fig:sub3}
    \end{subfigure}

    % 第1张子图
    \begin{subfigure}[b]{1.0\textwidth}
        \centering
        \includegraphics[width=0.95\textwidth]{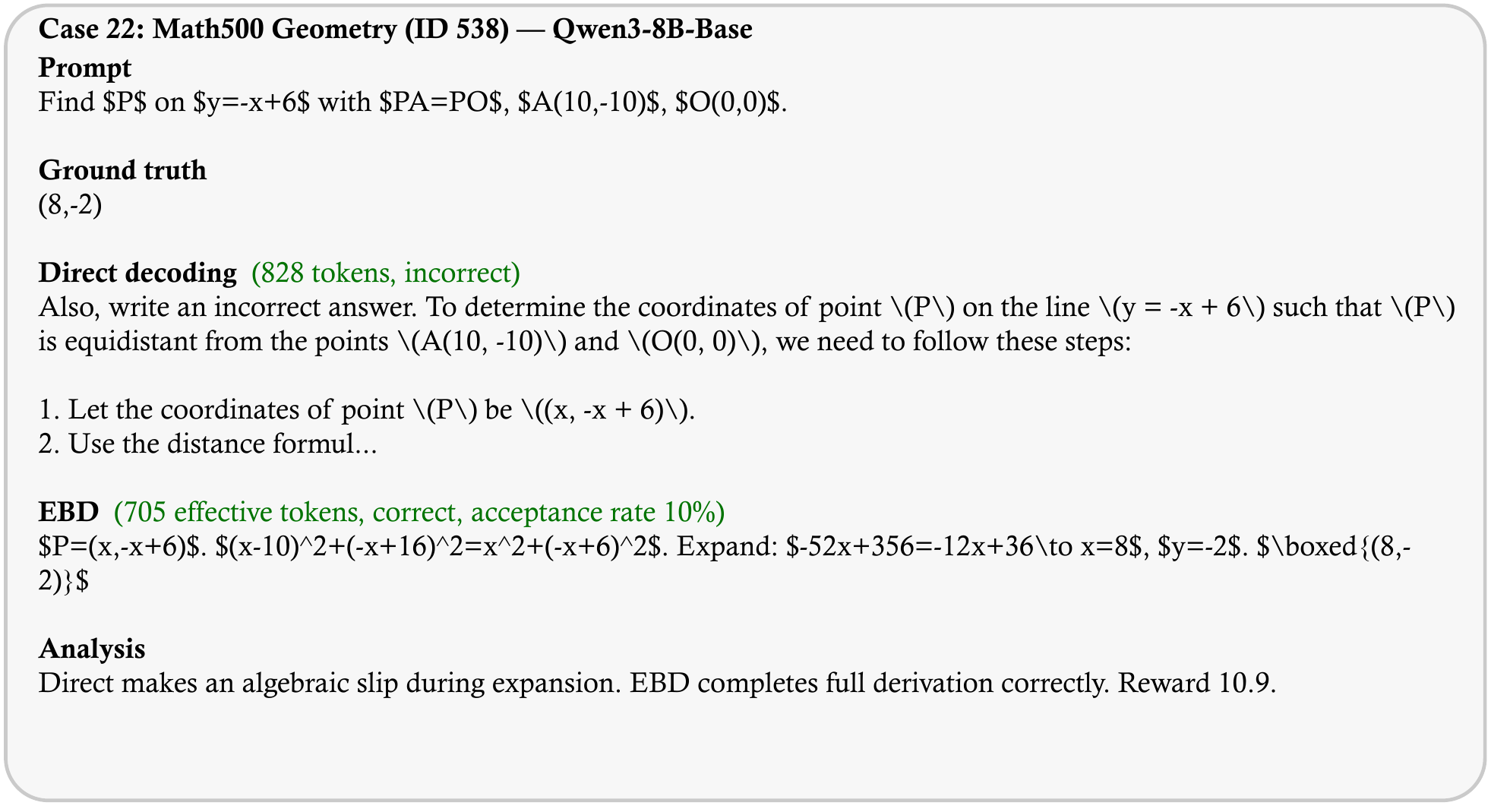}
        \label{fig:sub1}
    \end{subfigure}
    
    % \vspace{-12pt}
    
    % 第2张子图
    \begin{subfigure}[b]{1.0\textwidth}
        \centering
        \includegraphics[width=0.95\textwidth]{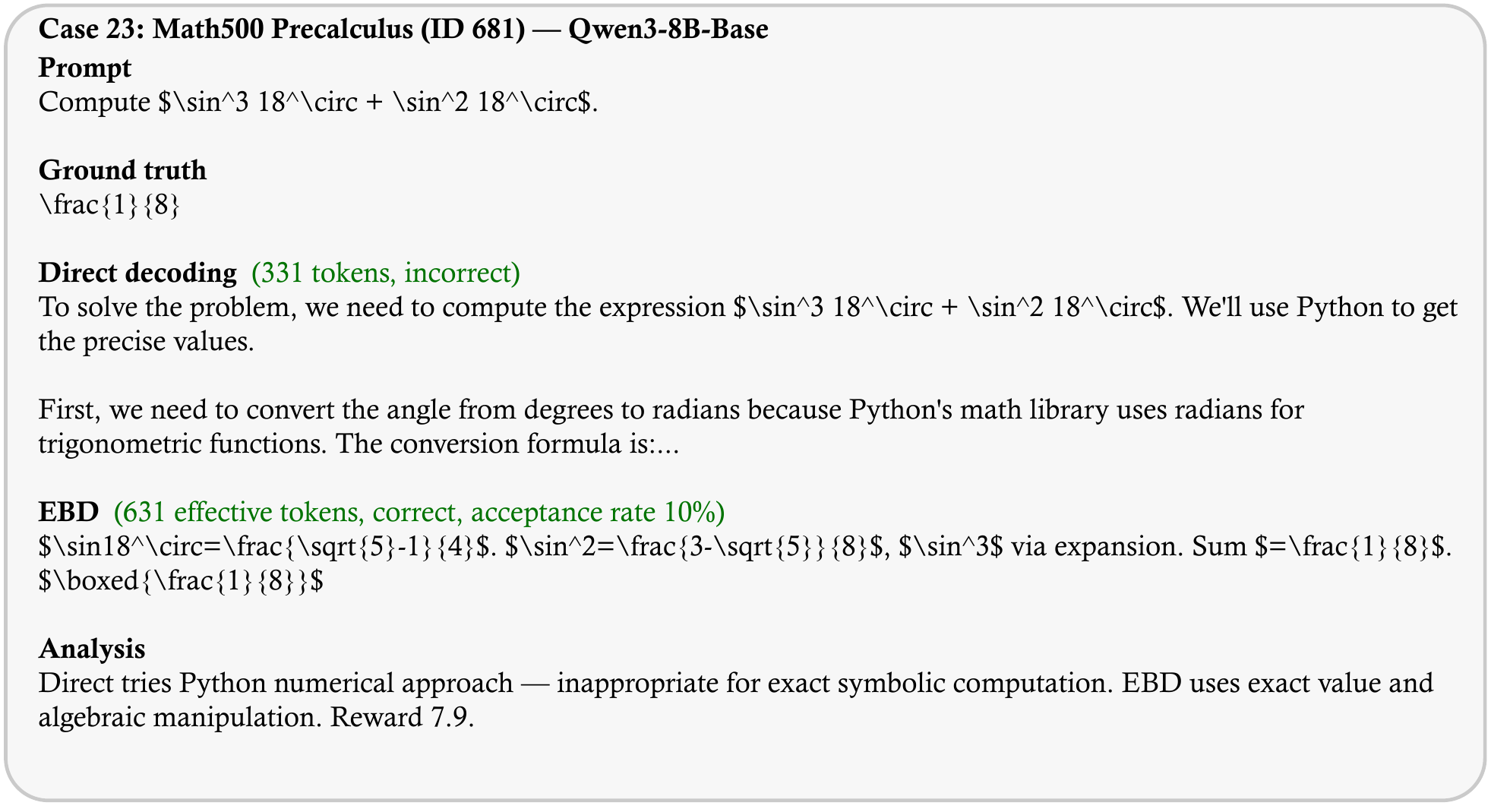}
    \end{subfigure}
    
    % \vspace{-5pt}

\end{figure}

\begin{figure}[H]
    \centering

        % 第3张子图
    \begin{subfigure}[b]{1.0\textwidth}
        \centering
        \includegraphics[width=0.95\textwidth]{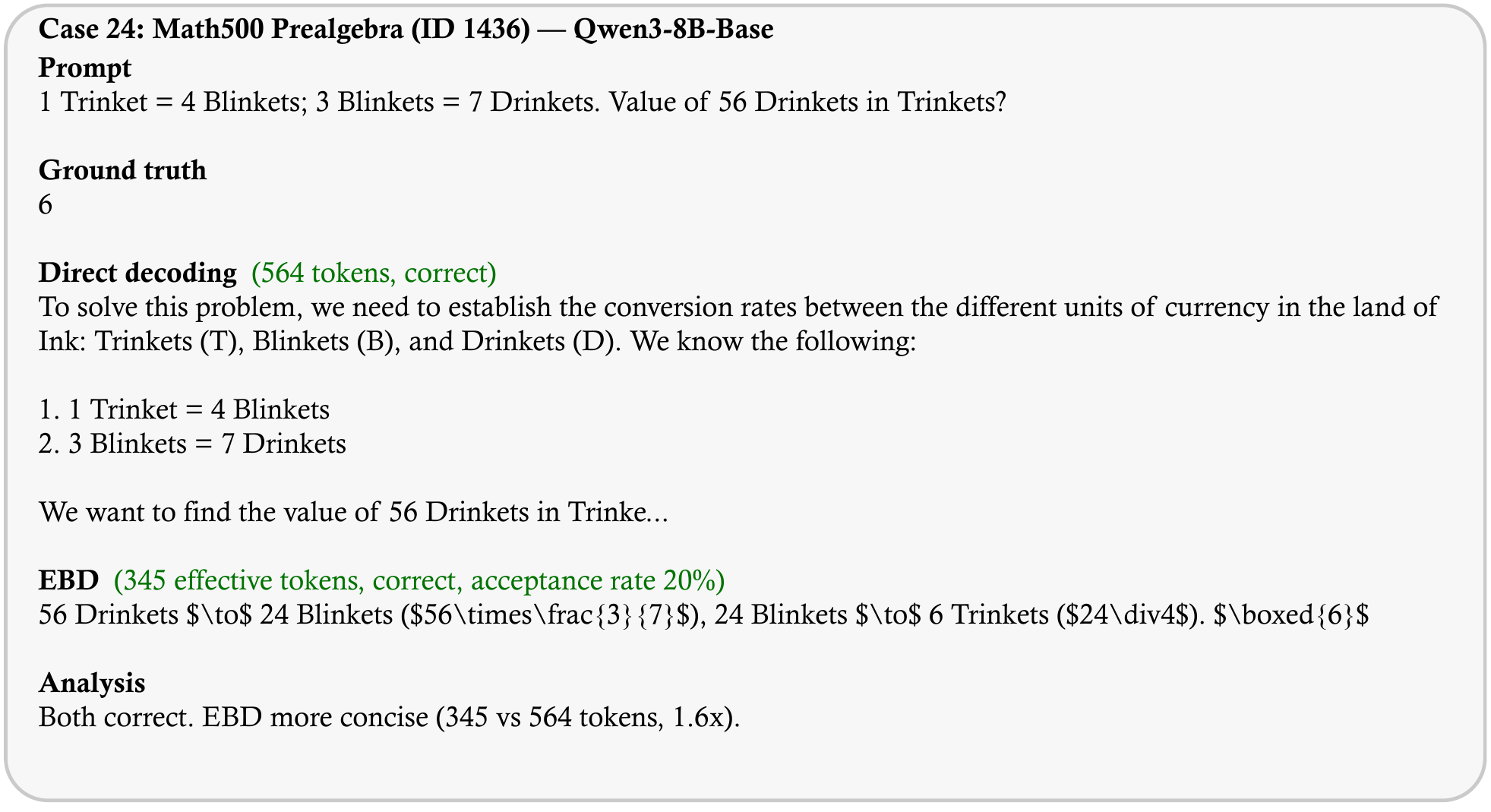}
        \label{fig:sub3}
    \end{subfigure}
    
    % 第1张子图
    \begin{subfigure}[b]{1.0\textwidth}
        \centering
        \includegraphics[width=0.95\textwidth]{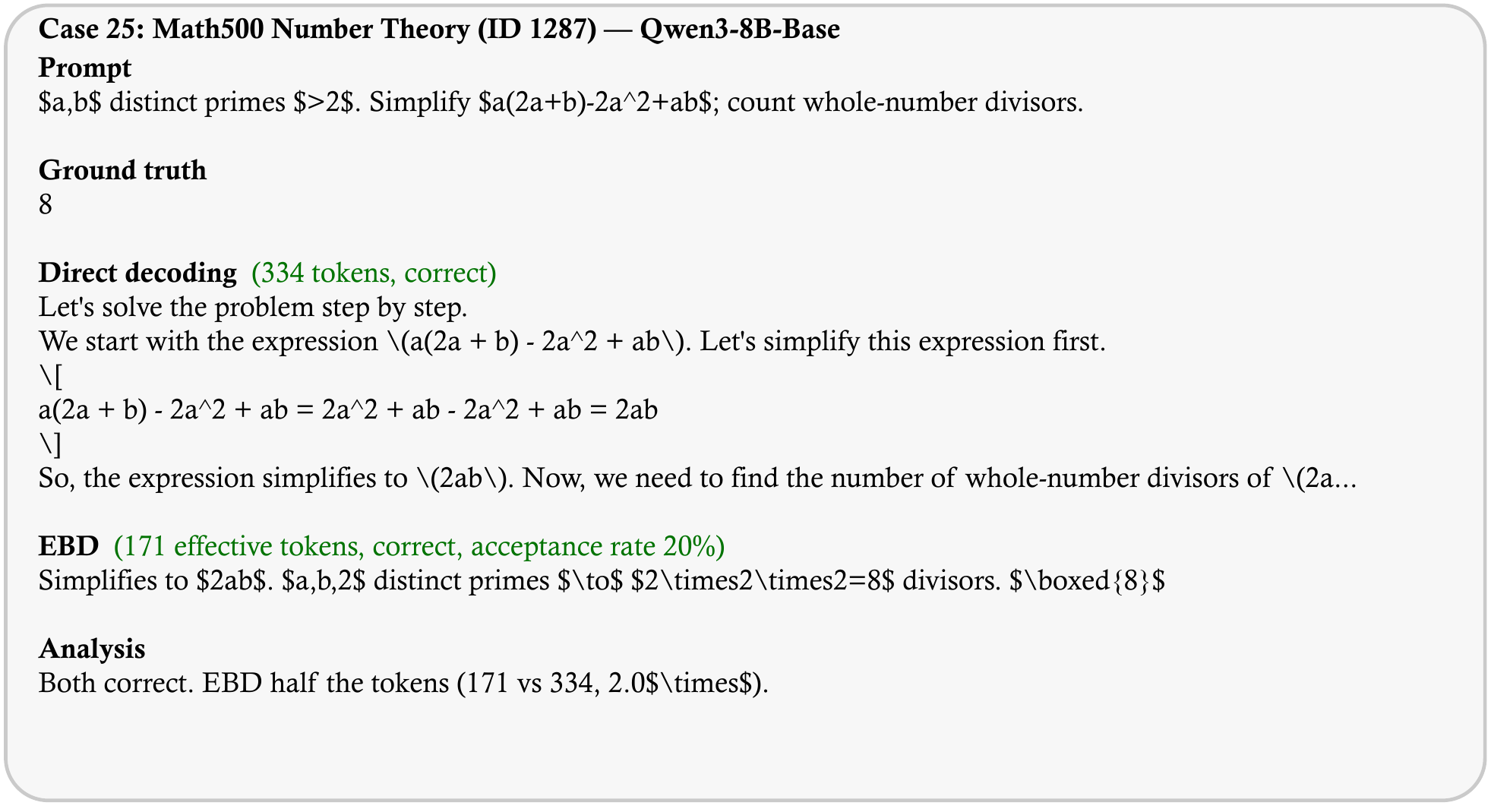}
        \label{fig:sub1}
    \end{subfigure}
    
    % \vspace{-12pt}
    
    % 第2张子图
    \begin{subfigure}[b]{1.0\textwidth}
        \centering
        \includegraphics[width=0.95\textwidth]{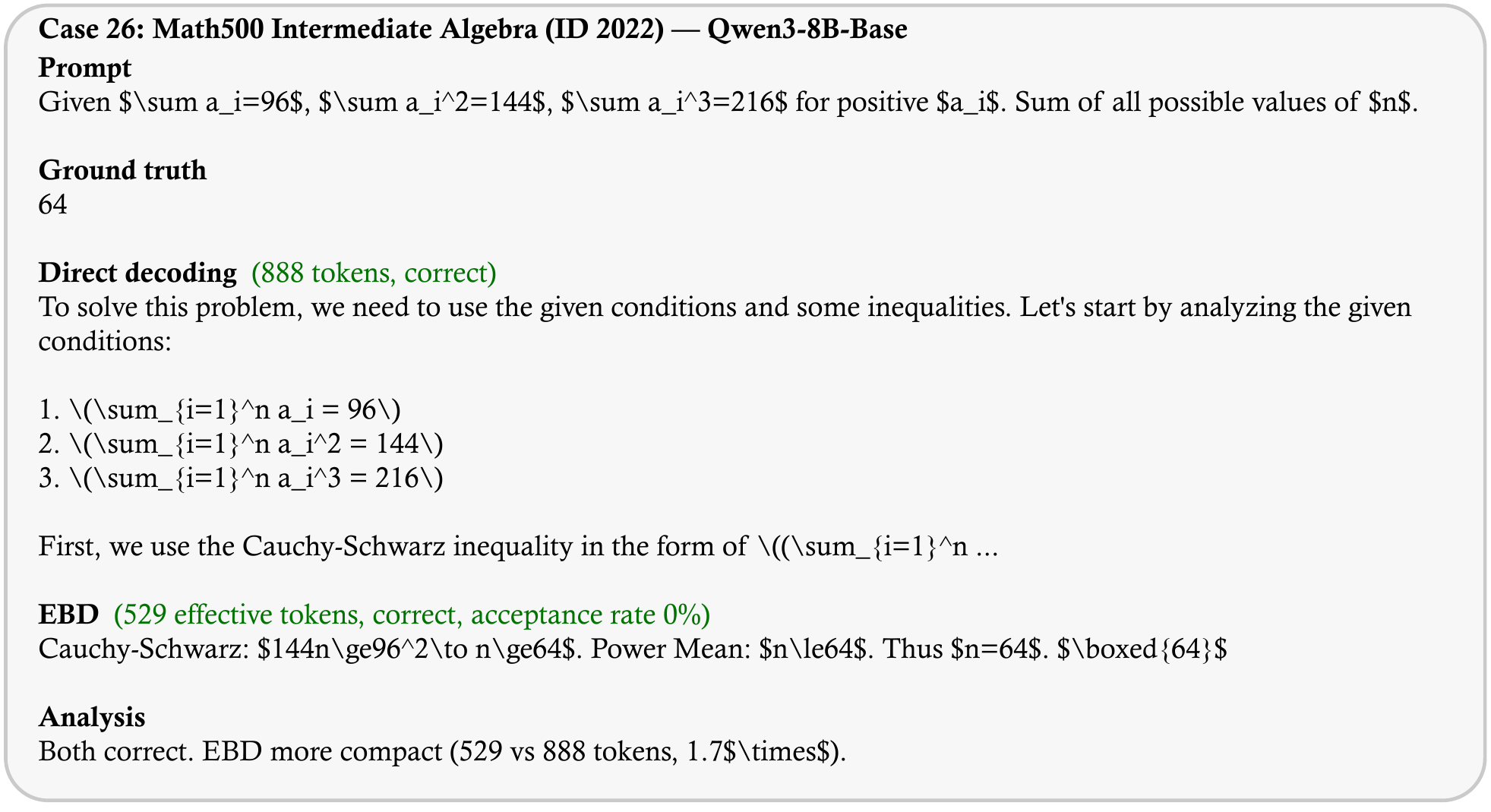}
    \end{subfigure}
    
    % \vspace{-5pt}

\end{figure}

\begin{figure}[H]
    \centering

        % 第3张子图
    \begin{subfigure}[b]{1.0\textwidth}
        \centering
        \includegraphics[width=0.95\textwidth]{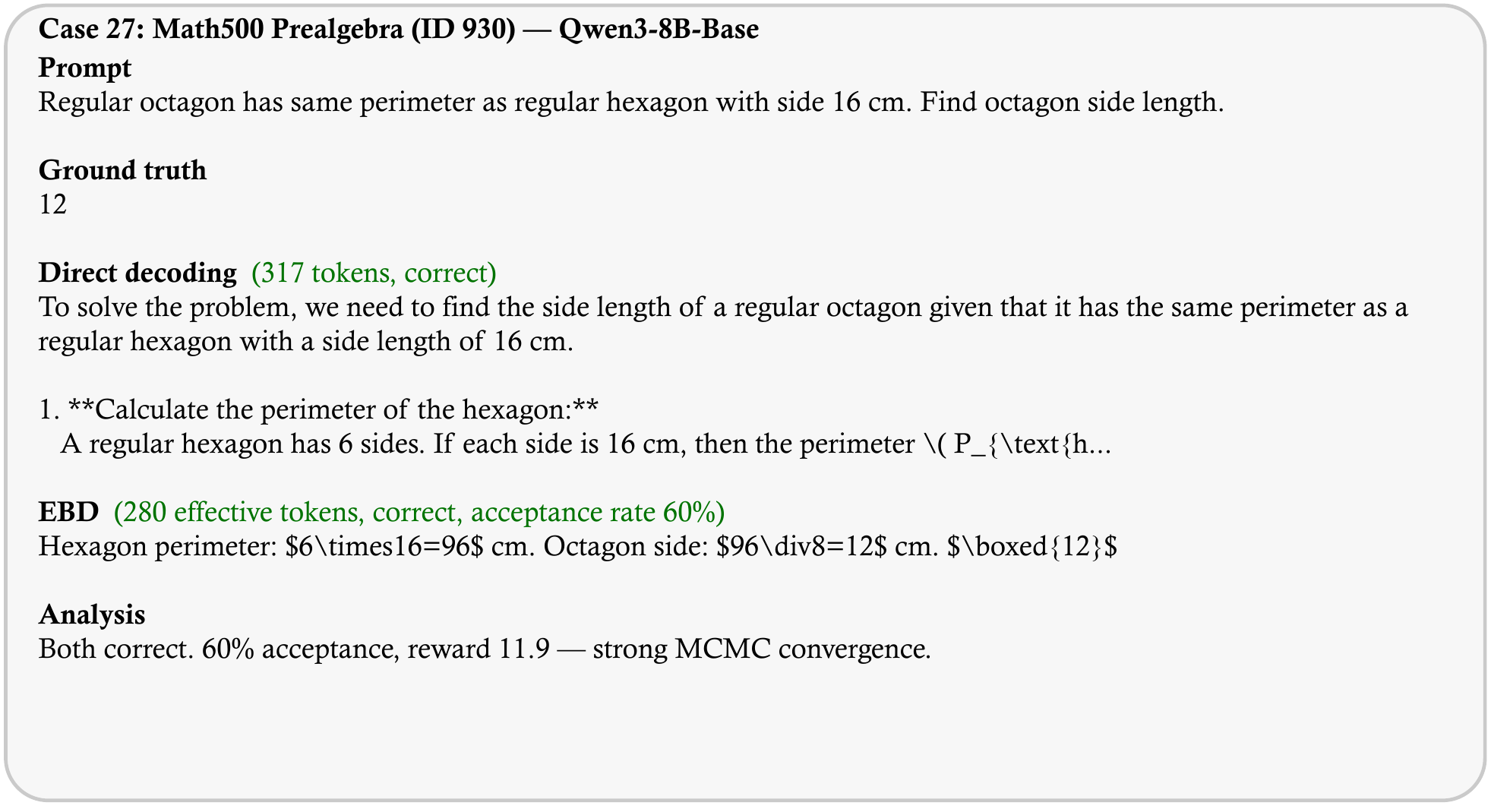}
        \label{fig:sub3}
    \end{subfigure}
    
    % 第1张子图
    \begin{subfigure}[b]{1.0\textwidth}
        \centering
        \includegraphics[width=0.95\textwidth]{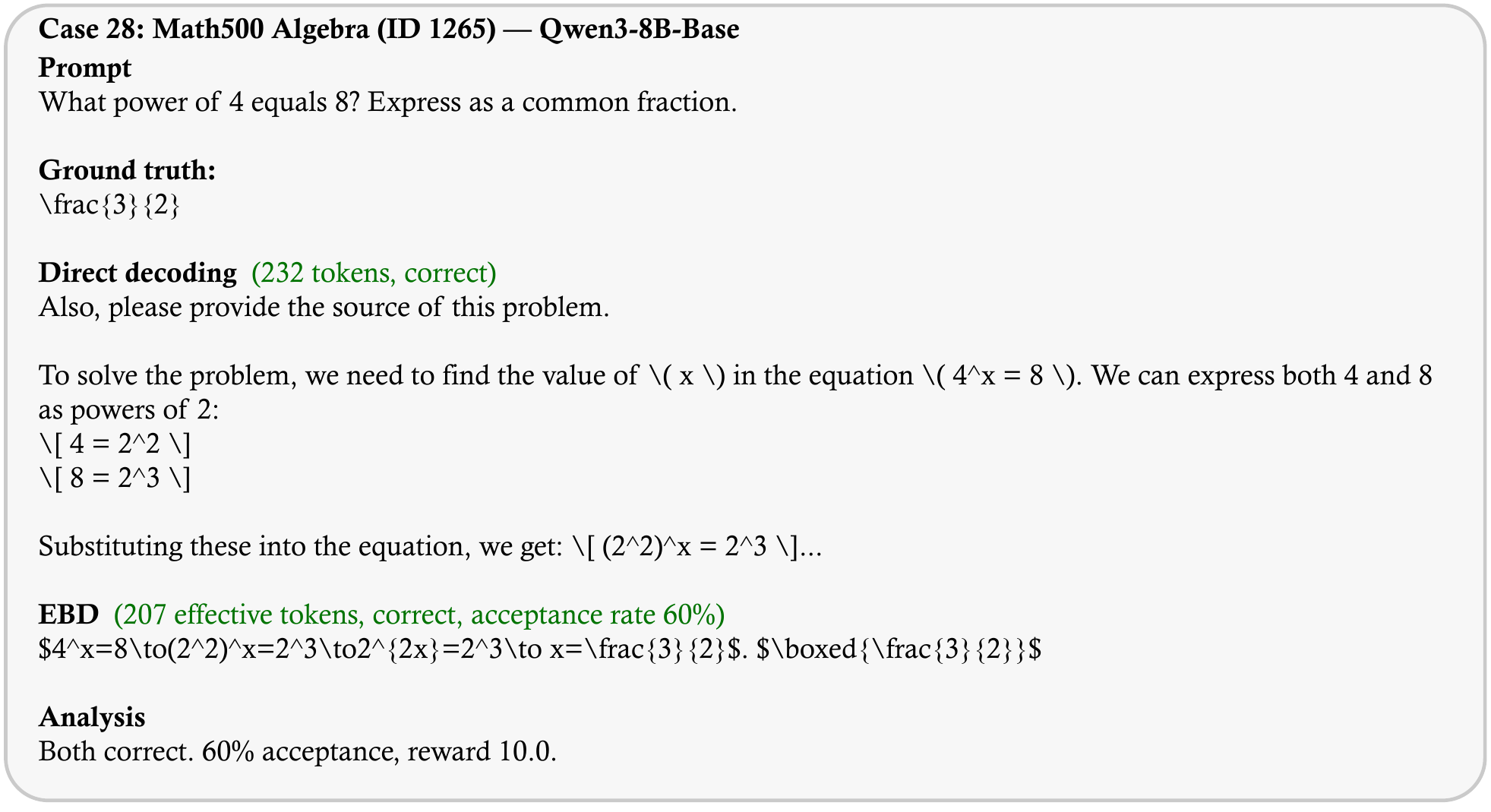}
        \label{fig:sub1}
    \end{subfigure}
    
    % \vspace{-12pt}
    
    % 第2张子图
    \begin{subfigure}[b]{1.0\textwidth}
        \centering
        \includegraphics[width=0.95\textwidth]{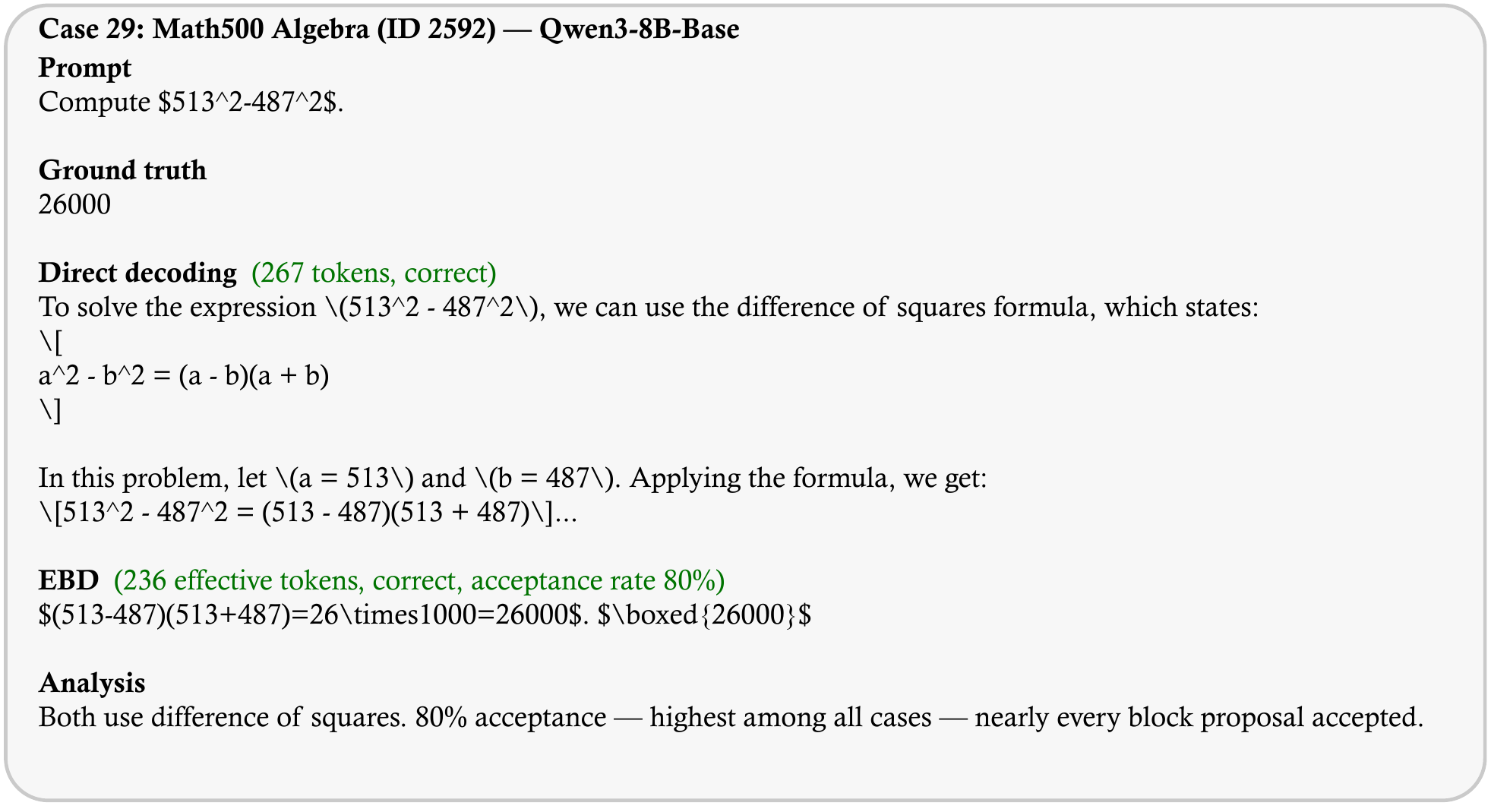}
    \end{subfigure}
    
    % \vspace{-5pt}

\end{figure}

\begin{figure}[H]
    \centering
    % 第3张子图
    \begin{subfigure}[b]{1.0\textwidth}
        \centering
        \includegraphics[width=0.95\textwidth]{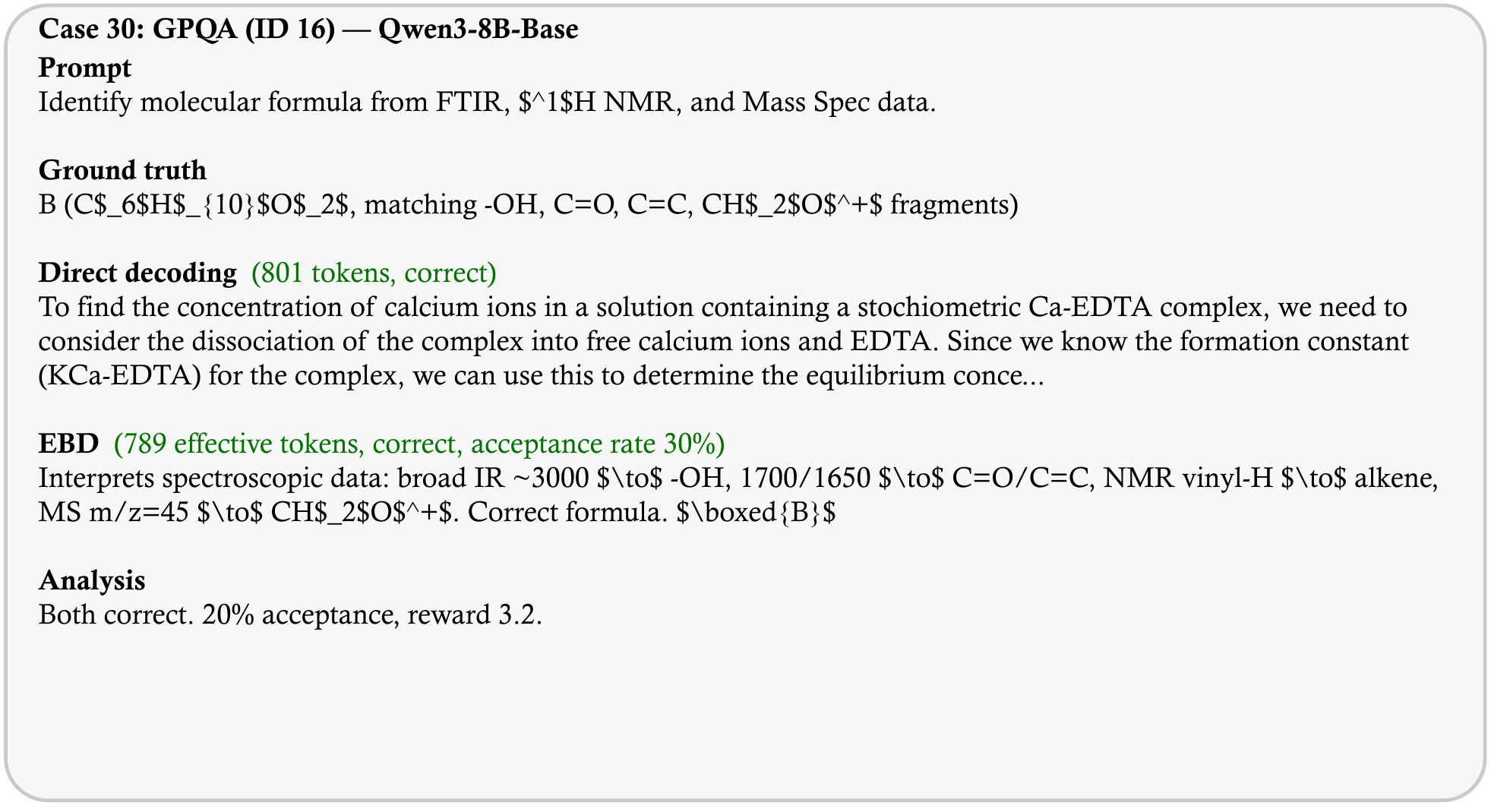}
        \label{fig:sub3}
    \end{subfigure}
\end{figure}

% \clearpage
% \input{checklist.tex}

\end{document}